\documentclass[twocolumn]{article}
\usepackage{fullpage}

\usepackage{hyperref}
\usepackage{amsmath, amsfonts, amsthm, mathtools}
\usepackage{graphicx}
\usepackage{subcaption}
\usepackage{pagecolor}

\newtheorem{prop}{Proposition}

\begin{document}

\setlength{\belowdisplayskip}{5pt} \setlength{\belowdisplayshortskip}{5pt}
\setlength{\abovedisplayskip}{5pt} \setlength{\abovedisplayshortskip}{5pt}

\setlength{\parskip}{3pt plus0pt minus0pt}

\twocolumn[

\title{Collective evolution of weights in wide neural networks}
\author{Dmitry Yarotsky\\
Skolkovo Institute of Science and Technology,  Moscow \\
\texttt{d.yarotsky@skoltech.ru}}
\date{}

\maketitle

]

\begin{abstract}
We derive a nonlinear integro-differential transport equation describing  collective evolution of weights under gradient descent in large-width neural-network-like models. 
We characterize stationary points of the evolution and analyze several scenarios where the transport equation can be solved approximately. We test our general method  in the special case of linear free-knot splines, and find good agreement between theory and experiment in observations of global optima, stability of stationary points, and convergence rates. 
\end{abstract}

\section{Introduction}
Modern neural networks include millions of neurons, and one can expect that macroscopic properties of such big systems can be described by analytic models -- similarly to how statistical mechanics or fluid dynamics describe macroscopic properties of big colections of physical particles. One specific approach in this general direction is to consider the limit of ``large-width'' neural networks, typically combined with the assumption of weight independence. There is a well-known connection between neural networks, Gaussian processes and kernel machines in this limit (\cite{neal2012bayesian, williams1997computing}). 
Recently, this limit has been used, along with some complex mathematical tools, to understand the landscape of the loss surface of large networks. In particular, \cite{choromanska2015loss} establish a link to the theory of spin glasses and use this theory to explain the distribition of critical points on the loss surface; \cite{pennington2017geometry} analyze the loss surface using random matrix theory; \cite{poole2016exponential, schoenholz2016deep} analyze propagation of information in a deep network assuming Gaussian signal distributions in wide layers. 

In the present paper we apply the large-width limit to describe the collective evolution of weights under standard gradient descent used to train the network. Weight optimization is currently a topic of active theoretical research. While local convergence of gradient descent is generally well-understood (\cite{nesterov2013introductory}), its global and statistical properties are hard to analyze due to the complex nonconvex shape of the loss surface. There are some special scenarios where the absence of spurious local minima has been proved, in particular in deep linear networks (\cite{kawaguchi2016deep, laurent2017deep}) or for pyramidal networks trained on small training set (\cite{nguyen2017loss}). But in general, the loss surface is known to have many local minima or saddle points trapping or delaying optimization, see e.g. (\cite{safran2017spurious, 
dauphin2014identifying}). Similarly, gradient descent is known to be analytically tractable in some cases, e.g. in linear networks (\cite{saxe2013exact, bartlett2018gradient}) or in ReLU networks under certain assumptions on the distribution of the inputs (\cite{tian2017analytical}), but in general,  current studies of gradient descent in large networks tend to postulate its replacement by an ``effective'' macroscopic model  (\cite{shwartz2017opening, chaudhari2018deep}). 

The goal of the present paper is to derive a general macroscopic description of the gradient descent directly from the underlying finite model, without ad hoc assumptions. We do this in section \ref{sec:transp}, obtaining a nonlinear integro-differential transport equation. In sections \ref{sec:properties},\ref{sec:mfloss} we analyze its general properties and characterize its global minima and stationary points. Interestingly, an important quantity for the transport equation is the quadratic mean-field version of the loss function: this loss does not increase under the evolution. However, the generator of the transport equation is not the formal gradient of this loss -- in particular, that is why the macroscopic dynamics does have stationary points that are not global minima. Next, in sections \ref{sec:linglobmin}--\ref{sec:localgauss} we analyze three scenarios where the transport equation can be solved perturbatively (near a global minimizer, for small output weights, and for localized Gaussian distributions). Finally, in section \ref{sec:splines} we verify our general method on a simple one-dimensional ReLU network (essentially describing linear splines). We compare the theoretical predictions involving stationary points, their stability, and convergence rate of the dynamics with the experiment and find a good agreement between the two.

This short paper is written at a ``physics level of rigor''. We strive to expose the main ideas and do not attempt to fill in all mathematical details. 

\section{General theory}
\subsection{A ``generalized shallow network'' model}
We consider the problem of approximating a ``ground truth'' map $f:X \to \mathbb R$. The nature of the set $X$ will not be important for the present  exposition.  We consider approximation by the ``generalized shallow network'' model:
\begin{equation}\label{eq:approx}\widehat f(\mathbf W, \mathbf x)=\frac{1}{N}\sum_{n=1}^N\phi(\mathbf w_n, \mathbf x), \quad \mathbf w_n\in\mathbb R^d, \mathbf x\in X,\end{equation}
where $\phi$ is some function of the weights and the input, $\mathbf w_n$ is the weight vector for the $n$'th term, and $\mathbf W = (\mathbf w_1,\ldots, \mathbf w_N)\in \mathbb R^{Nd}$ is the full weight vector. The usual neural network with a single hidden layer is obtained when $X\subset \mathbb R^\nu, d=\nu+2, \mathbf w=(c,w_1,\ldots,w_\nu,h)$ and $\phi(\mathbf w, \mathbf x)=c\sigma(\sum_{n=1}^\nu w_nx_n+h)$ with some nonlinear activation function $\sigma$.\footnote{The standard network model also includes a constant term in the output layer, but we will omit it for simplicity of the exposition.}  

We consider the usual quadratic loss function
\begin{equation}\label{eq:loss}L(\mathbf W)=\frac{1}{2}\int_X (\widehat f(\mathbf W, \mathbf x)-f(\mathbf x))^2 d\mu(\mathbf x),\end{equation}
where $\mu$ is some measure on $X$. Again, we don't assume anything specific about $\mu$: for example, it can consist of finitely many atoms (a ``finite training set'' scenario) or have a density w.r.t. the Lebesgue measure (assuming $X\subset\mathbb R^\nu$; a ``population average'' scenario).

We will study the standard gradient descent dynamics of the weights: 
\begin{equation}\label{eq:dtw}\frac{d}{dt} \mathbf W = \mathbf G,\quad \mathbf G=-\alpha\nabla_{\mathbf W} L,\end{equation}
where $\alpha$ is the learning rate and $\nabla_{\mathbf W}$ is the gradient w.r.t. $\mathbf W$. 

\subsection{Derivation of the transport equation } \label{sec:transp} We are interested in the collective behavior of the weights under gradient descent in the ``wide network'' limit $N\to\infty$. It is convenient to view the weight vector $\mathbf W$ as a stochastic object described by a distribution with a  density function $P(\mathbf W, t)\ge 0$ such that $\int_{\mathbb R^{Nd}}P(\mathbf W, t)d\mathbf W=1$. The common practice in network training (see e.g. \cite{glorot2010understanding}) is to initialize the weights $\mathbf w_n$ randomly and independently with some density $p(\cdot,0)$, i.e. $P(\mathbf W, 0)=\prod_{n=1}^Np(\mathbf w_n,0).$ Our immediate goal will be to derive, under suitable assumptions, the equation governing the evolution of the factors.

First we replace the system \eqref{eq:dtw} of ordinary differential equations  in $\mathbf W$  by a partial differential equation in $P$. We can view the vector field $\mathbf G=\mathbf G(\mathbf W)$ in Eq.\eqref{eq:dtw} as a ``probability flux''. Then, the evolution of $P$ can be described by the continuity equation expressing local probability conservation: \begin{equation}\label{eq:conteq}\frac{1}{P}\frac{d}{dt}P=-\nabla_{\mathbf W}\cdot\mathbf G=-\nabla_{\mathbf W}\cdot(-\alpha\nabla_{\mathbf W} L)=\alpha\Delta_{\mathbf W} L,\end{equation}
where $\Delta_{\mathbf W}$ is the Laplacian and $\frac{d}{dt}P$ denotes the material derivative: $$\frac{d}{dt}P=\frac{\partial}{\partial t}P+\frac{d}{dt} \mathbf W\cdot\nabla_{\mathbf W}P$$
(see Appendix \ref{sec:conteq} for details).

Expanding the continuity equation, we get
\begin{align}\frac{1}{P}\frac{dP}{dt}(\mathbf W)=&\alpha\Delta L(\mathbf W)=\alpha\sum_{n=1}^N\Delta_{\mathbf w_n} L(\mathbf W)\nonumber\\
=&\frac{\alpha}{N}\int_X\Big[ (\widehat f(\mathbf W, \mathbf x)-f(\mathbf x))\sum_{n=1}^N\Delta_{\mathbf w} \phi({\mathbf w}_n, \mathbf x)\nonumber\\ \label{eq:1pdt}
&+\frac{1}{N}\sum_{n=1}^N|\nabla_{\mathbf w} \phi({\mathbf w}_n, \mathbf x)|^2\Big] d\mu(\mathbf x)
\end{align}
We look for factorized solutions $P(\mathbf W,t)=\prod_{n=1}^N p(\mathbf w_n,t)$
with some density $p$. However, it is clear that we cannot fulfill this factorization exactly because of the interaction of different weights $\mathbf w_n$ on the r.h.s. To decouple them, we perform the ``mean field'' (or ``law of large numbers'') approximation:
\begin{align}\widehat f(\mathbf W, \mathbf x)={}&\frac{1}{N}\sum_{n=1}^N\phi(\mathbf w_n, \mathbf x)\label{eq:mf}\\\approx{}& \mathbb E_{\mathbf w\sim p}\phi(\mathbf w,\mathbf x)=\int p(\mathbf w,t)\phi(\mathbf w,\mathbf x)d\mathbf w.\nonumber\end{align}
This approximation corresponds to the limit $N\to\infty$. To obtain a finite non-vanishing r.h.s. in Eq.\eqref{eq:1pdt}, we rescale the learning rate $\alpha$ by setting $\alpha=N$. We discard the second term $\frac{1}{N}\sum|\nabla_{\mathbf w}\phi|^2$ on the r.h.s. of \eqref{eq:1pdt}, since the coefficient $\frac{1}{N}$ makes it asymptotically small compared to the first term. After all this, we obtain the equation  
\begin{align*}\frac{1}{P}\frac{dP}{dt}(\mathbf W)={}&\int_X\Big(\int  p(\mathbf w',t)\phi(\mathbf w',\mathbf x)d\mathbf w'-f(\mathbf x)\Big)\\
&\times\sum_{n=1}^N\Delta_{\mathbf w} \phi({\mathbf w_n}, \mathbf x) d\mu(\mathbf x),\end{align*}
admitting a factorized solution. Namely, using the product form of $P$, unwrapping the material derivative and equating same-$n$ terms, we get
\begin{align}
\frac{1}{p(\mathbf w_n,t)}&\Big(\frac{\partial}{\partial t}p(\mathbf w_n,t)+\frac{d}{dt}\mathbf w_n\cdot\nabla_\mathbf{w}p(\mathbf w_n,t)\Big)\nonumber\\
={}&\int_X\Big(\int  p(\mathbf w',t)\phi(\mathbf w',\mathbf x)d\mathbf w'-f(\mathbf x)\Big)\nonumber\\
&\times\Delta_{\mathbf w} \phi({\mathbf w_n}, \mathbf x) d\mu(\mathbf x).\label{eq:1pwn}
\end{align}
Using the mean-field approximation \eqref{eq:mf}, we can write
\begin{align*}\frac{d}{dt} \mathbf w_n ={}& -N\nabla_{{\mathbf w}_n} L\\
={}&-\int_X (\widehat f(\mathbf W, \mathbf x)-f(\mathbf x))\nabla_{\mathbf w} \phi({\mathbf w}_n, \mathbf x) d\mu(\mathbf x).\end{align*}
Making this replacement for $\frac{d}{dt} \mathbf w_n$ in Eq.\eqref{eq:1pwn}, dropping the index $n$ and rearranging the terms, we obtain the desired \emph{transport equation} for $p$: 
\begin{align}\frac{\partial}{\partial t} p(\mathbf w, t)
= {}& \int_X\Big(\int  p(\mathbf w',t)\phi(\mathbf w',\mathbf x)d\mathbf w'-f(\mathbf x)\Big) \nonumber\\
&\times \nabla_{\mathbf w}\cdot( p(\mathbf w,t)\nabla_{\mathbf w} \phi({\mathbf w}, \mathbf x))d\mu(\mathbf x)\label{eq:transp}.\end{align}

\subsection{Properties of the transport equation}\label{sec:properties}
The derived equation \eqref{eq:transp} is a nonlinear integro-differential equation in $p(\mathbf w,t),$
and we expect it to approximately describe the evolution of the distribution of the weights $\mathbf w_n$. In agreement with this interpretation, the equation \eqref{eq:transp} preserves the total probability ($\frac{d}{dt}\int p(\mathbf w,t)d\mathbf w=0$) -- this follows from the gradient form of the r.h.s. of Eq.\eqref{eq:transp} (assuming the function $p(\mathbf w,t)$ falls off sufficiently fast as $|\mathbf w|\to\infty$, so that the boundary term can be omitted when integrating by parts). We also expect that, under suitable regularity assumptions, the equation preserves the nonnegativity $p(\mathbf w,0)\ge 0$ of the initial condition. In the sequel, our treatment of this equation will be rather heuristic; in particular, we will not distinguish between regular and weak (distributional) solutions $p$.      

It is helpful to consider Eq.\eqref{eq:transp} as a pair of an integral and a differential equations:
\begin{align}u(\mathbf w,t)={}&\int_X\Big(\int p(\mathbf w',t)\phi(\mathbf w',\mathbf x)d\mathbf w'-f(\mathbf x)\Big)\nonumber\\
&\times\phi({\mathbf w}, \mathbf x)d\mu(\mathbf x),\label{eq:transp21}\\
\frac{\partial}{\partial t} p(\mathbf w, t)={}&\nabla_{\mathbf w}\cdot( p(\mathbf w,t)\nabla_{\mathbf w} u(\mathbf w,t)).\label{eq:transp22}\end{align}
The quantity $u(\mathbf w,t)$ reflects the correlation of the current approximation error, as a function of $\mathbf x$, with the function $\phi(\mathbf w,\cdot)$. If $u(\mathbf w,t)$ is known for all $\mathbf w,t$, then equation \eqref{eq:transp22}, being  first-order in $p$, can be solved by the method of characteristics. Specifically, let $\mathbf w(t)$ be some trajectory satisfying the equation $\frac{d}{dt}\mathbf w=-\nabla_{\mathbf w} u(\mathbf w,t)$. On this trajectory, by Eq.\eqref{eq:transp22}, 
\[\frac{d}{dt}p(\mathbf w(t),t)=p(\mathbf w(t),t)\Delta_\mathbf{w}u(\mathbf w(t),t),\]
which has the solution
\[p(\mathbf w(t),t)=p(\mathbf w(t_0),t_0)e^{\int_{t_0}^t \Delta_\mathbf{w}u(\mathbf w(\tau),\tau)d\tau}\]
This shows that, geometrically, $-\nabla_{\mathbf w}u(\mathbf w,t)$ determines the transport direction in the $\mathbf w$ space, while $\Delta_{\mathbf w}u(\mathbf w,t)$ determines the infinitesimal change of the value of $p$.

\subsection{The mean-field loss and stationary points}\label{sec:mfloss}
We define the mean-feald loss function on distributions $p(\mathbf w)$ by plugging the approximation \eqref{eq:mf} into the original loss formula \eqref{eq:loss}:  
\begin{equation}\label{eq:lmfdef}L_{\mathrm{mf}}(p)=\frac{1}{2}\int_X \Big(\int p(\mathbf w)\phi(\mathbf w,\mathbf x)d\mathbf w-f(\mathbf x)\Big)^2 d\mu(\mathbf x).\end{equation}
If $p(\mathbf w,t)$ is a solution of the transport equation \eqref{eq:transp}, then, by a computation involving integration by parts,
\begin{equation}\label{eq:dtlossmf}
\frac{d}{dt}L_{\mathrm{mf}}(p(\cdot,t))=-\int p(\mathbf w,t)|\nabla_{\mathbf w}u(\mathbf w,t)|^2d\mathbf w,
\end{equation}
where $u(\mathbf w,t)$ is defined by Eq.\eqref{eq:transp21} (see Appendix \ref{sec:derivloss}). 
In particular, the mean-field loss does not decrease for any nonnegative solution $p$:  \begin{equation}\label{eq:monot}\frac{d}{dt}L_{\mathrm{mf}}(p(\cdot,t))\le 0.
\end{equation} This property is inherited from the original gradient descent \eqref{eq:dtw}. Note, however, that the transport equation \eqref{eq:transp} is \emph{not} a (formal) gradient descent performed on $L_{\mathrm{mf}}$ w.r.t. $p$ -- this latter would be given by $\frac{\partial}{\partial t} p(\mathbf w, t)=[-\nabla_{p}L_{\mathrm {mf}}(p)](\mathbf w, t)=-u(\mathbf w, t)$ and is quite different from Eq.\eqref{eq:transp}. (In fact, in contrast to Eq.\eqref{eq:transp}, the equation $\frac{\partial}{\partial t} p(\mathbf w, t)=-u(\mathbf w, t)$ does not seem to generally admit reasonable solutions --  even in simplest cases solutions may blow up in infinitesimal time). 

Accordingly, the intuition behind the connection of the transport equation \eqref{eq:transp} to the mean-field loss is different from the na\"ive intuition of gradient descent. Note that $L_{\mathrm{mf}}$ is a convex quadratic functional in $p,$ and the set of all nonnegative distributions $p$ is convex. The intuition of finite-dimensional gradient descent then suggests that the solution should invariably converge to a global minimum of $L_{\mathrm{mf}}$. However, this is not the case with the dynamics \eqref{eq:transp} which can easily get trapped at stationary points. The meaning of the stationary point is clarified by the following proposition.

\begin{prop}\label{prop:equiv} Assuming $p(\mathbf w, t)\ge 0$ is a solution of Eq.\eqref{eq:transp}, the following conditions are equivalent at any particular $t=t_0$:
\begin{enumerate}
\item $\frac{d}{dt}L_{\mathrm{mf}}(p(\cdot,t_0)) = 0$;
\item $\nabla_{\mathbf w} u(\mathbf w,t_0)=0$ at all $\mathbf w$ where $p(\mathbf w,t_0)>0$. 
\item $\frac{\partial p}{\partial t}(\cdot, t_0)\equiv 0$.
\end{enumerate} 
\end{prop}
\begin{proof}
The equivalence 1) $\Leftrightarrow$ 2) follows from Eq.\eqref{eq:dtlossmf}. The implication 3)  $\Rightarrow$ 1) is trivial. It remains to establish 2)  $\Rightarrow$ 3). Since $p\ge 0$, condition 2) implies that $p(\mathbf w,t_0) \nabla_{\mathbf w} u(\mathbf w,t_0)\equiv 0$ for all $\mathbf w$. Then, by Eq.\eqref{eq:transp22},
$\frac{\partial}{\partial t} p(\mathbf w, t_0)=  \nabla_{\mathbf w}( p(\mathbf w,t_0) \nabla_{\mathbf w} u(\mathbf w,t_0))=0$ for all $\mathbf w$.
\end{proof}
Thus, a stationary distribution $p(\mathbf w)$ can be characterized by any of the three conditions of this proposition. Given the monotonicity \eqref{eq:monot}, the proposition suggests that, under reasonable regularity assumptions, any solution of the transport equation \eqref{eq:transp} eventually converges to such a stationary distribution.

If the family of maps $\{\phi(\mathbf w, \cdot)\}_{\mathbf w\in\mathbb R^d}$ is sufficiently rich, then  we expect a stationary distribution $p$ either to be a global minimizer of $L_{\mathrm {mf}}$ or to be  concentrated on some proper subset of the $\mathbf w$-space $\mathbb R^d$. We state one particular proposition demonstrating this in the important case of \emph{models with linear output}. We say that an approximation \eqref{eq:approx} is a model with linear output if $\phi(\mathbf w,\mathbf x)=w_0\widetilde{\phi}(\widetilde{\mathbf w},\mathbf x)$, where $\mathbf w=(w_0, \widetilde{\mathbf w})$ and $\widetilde{\phi}$ is some map. For example, the standard shallow neural network is a model with linear output. 
For the proposition below, we assume that the functions $\widetilde{\phi}(\widetilde{\mathbf w}, \cdot)$ and the error function $\mathbf x\mapsto \int p(\mathbf w)\phi(\mathbf w,\mathbf x)d\mathbf w-f(\mathbf x)$ are square-integrable w.r.t. the measure $\mu$.  
Recall that a set of vectors in a Hilbert space is called \emph{total} if their finite linear combinations are dense in this space.  
\begin{prop}  Let $p(\mathbf w)$ be a stationary distribution supported on a subset $A\subset \mathbb R^d$. Suppose that the functions $\{\widetilde{\phi}(\widetilde{\mathbf w},\cdot)\}_{{\mathbf w}\in A }$ are total in $L^2(\mu)$. Then $L_{\mathrm{mf}}(p)=0.$  
\end{prop}
\begin{proof}
By the second stationarity condition of proposition \ref{prop:equiv}, we have $\nabla_{\mathbf w}u(\mathbf w)=0$ for all $\mathbf w\in A$. In particular, taking the partial derivative w.r.t. $w_0$, \[\int_X\Big(\int p(\mathbf w',t)\phi(\mathbf w',\mathbf x)d\mathbf w'-f(\mathbf x)\Big)\widetilde{\phi}(\widetilde{\mathbf w}, \mathbf x)d\mu(\mathbf x)=0\]
for all  $\mathbf w\in A$. Since the functions $\{\widetilde{\phi}(\widetilde{\mathbf w},\cdot)\}_{{\mathbf w}\in A }$ are total, $\int p(\mathbf w',t)\phi(\mathbf w',\mathbf x)d\mathbf w'=f(\mathbf x)$ almost everywhere w.r.t. $\mu$, and hence $L_{\mathrm{mf}}(p)=0.$  
\end{proof}

\subsection{Linearization near a global minimizer}\label{sec:linglobmin}
Suppose that the solution $p(\mathbf w,t)$ of the transport equation converges to a limiting distribution $p_\infty(\mathbf w)$ such that $L_{\mathrm{mf}}(p_\infty)=0,$ i.e. $\int p_\infty(\mathbf w',t)\phi(\mathbf w',\mathbf x)d\mathbf w'=f(\mathbf x)$ ($\mu$-a.e.). Consider the positive semidefinite kernel
\[K(\mathbf w,\mathbf w')=\int_X \phi(\mathbf w,\mathbf x)\phi(\mathbf w',\mathbf x)d\mu(\mathbf x)\]
and the corresponding integral operator
\[\mathcal K p(\mathbf w)=\int K(\mathbf w,\mathbf w') p(\mathbf w')d\mathbf w'.\]
Note that the operator $\mathcal K$ determines the quadratic part of the mean-field loss:    
\begin{equation}\label{eq:kmfl} L_{\mathrm{mf}}(p_\infty+\delta p)=\frac{1}{2}\langle \mathcal K\delta p, \delta p\rangle,\end{equation}
where $\langle f,g\rangle=\int f(\mathbf w)g(\mathbf w)d\mathbf w.$

Let us represent the solution $p$ in the form $p(\mathbf w,t)=p_\infty(\mathbf w)+\delta p(\mathbf w,t)$ with a small $\delta p$. Plugging this into Eq.\eqref{eq:transp} and keeping only terms linear in $\delta p$, we obtain
\begin{align}\label{eq:r}\frac{\partial}{\partial t} \delta p(\mathbf w, t)\approx {}& \int_X\Big(\int  \delta p(\mathbf w',t)\phi(\mathbf w',\mathbf x)d\mathbf w')\Big) \nonumber\\
&\times \nabla_{\mathbf w}\cdot\big(p_\infty(\mathbf w)\nabla_{\mathbf w} \phi({\mathbf w}, \mathbf x)\big)d\mu(\mathbf x) \nonumber\\
= {}& \nabla\cdot\big(p_\infty\nabla) \mathcal K \delta p(\mathbf w, t)\nonumber\\
={}& \mathcal R \delta p(\mathbf w, t),
\end{align}
where $\mathcal R$ is the integral operator with the kernel
\begin{equation}\label{eq:rk}R(\mathbf w,\mathbf w')=\nabla_{\mathbf w}\cdot(p_\infty(\mathbf w)\nabla_{\mathbf w}) K(\mathbf w,\mathbf w').\end{equation}
Under this linearization, the solution to the transport equation can be written as 
\begin{equation}\label{eq:etr}
\delta p(\cdot, t) \approx e^{t\mathcal R}\delta p(\cdot, 0).
\end{equation}
The operator $\mathcal R$ is symmetric and negative semi-definite w.r.t. the semi-definite scalar product associated with the kernel $K$:
\[\langle \mathcal K p,\mathcal R q\rangle = \langle \mathcal K \mathcal R p, q\rangle = -\langle p_\infty \nabla \mathcal K p, \nabla \mathcal K q\rangle. \]
This suggests that we can use the spectral theory of self-adjoint operators to analize the large--$t$ evolution of $p(\cdot,t)$. A caveat here is that the scalar product $\langle \mathcal K \cdot, \cdot\rangle$ is not strictly positive definite, in general. This issue can be addressed by considering  the quotient space $\mathcal H_{K}'=\mathcal H_{K}/\mathcal N$, where $\mathcal H_{K}$ is the Hilbert space  associated with the semi-definite scalar product $\langle \mathcal K \cdot, \cdot\rangle$, and $\mathcal N=\{p:\mathcal K p=0\}=\{p:\langle \mathcal K p, p\rangle=0\}$. On this quotient space the scalar product $\langle \mathcal K \cdot, \cdot\rangle$ is non-degenerate, and $\mathcal R$ extends from $\mathcal H_{K}$ to $\mathcal H_{K}'$ since the subspace $\mathcal N$ lies in the null-space of $\mathcal R$.

Now, the self-adjoint operator $\mathcal R$ has an orthogonal spectral decomposition in $\mathcal H_K'$.   In particular, suppose that $\mathcal R$ has a discrete spectrum of eigenvalues $\lambda_k$ and $\langle\mathcal K\cdot,\cdot\rangle$--orthogonal eigenvectors $[p_k]\in\mathcal H_{K}'$. For each $k$, the equivalence class $[p_k]$ has a unique representative $p_k\in \mathcal H_{K}$ such that $\mathcal R p_k=\lambda_k p_k$ in $\mathcal H_{K}$. On the other hand, $\mathcal R\mathcal N=0$. We can view the space $\mathcal H_{K}$ as a direct sum $\mathcal H_{K}'\oplus \mathcal N,$ so this gives us the full eigendecomposition of the space $\mathcal H_{K}$. Recall that we have assumed that $p(\cdot,t)=p_\infty+\delta p(\cdot,t)\stackrel{t\to+\infty}{\longrightarrow} p_\infty$. By Eq.\eqref{eq:etr}, this means that the eigendecomposition of the vector $\delta p(\cdot,0)$ can include only eigenvectors with stricly negative eigenvalues, and in particular it does not include an $\mathcal N$--component. We can then write, by Eq.\eqref{eq:etr} and the eigenvector expansion,    
\begin{equation*}
\delta p(\cdot,t)\approx \sum_{k}\frac{\langle\mathcal K \delta p(\cdot,0),p_k\rangle}{\langle\mathcal K p_k,p_k\rangle}e^{\lambda_kt}p_k(\cdot)
\end{equation*}
and, by Eq.\eqref{eq:kmfl},
\begin{equation}\label{eq:lmflin}
L_{\mathrm{mf}}(p(\cdot,t))\approx\frac{1}{2}\sum_{k}\frac{\langle\mathcal K \delta p(\cdot,0),p_k\rangle^2}{\langle\mathcal K p_k,p_k\rangle}e^{2\lambda_kt}.
\end{equation}
We see, in particular, that the large-$t$ asymptotic of loss is determined by the distribution of the eigenvalues of $\mathcal K$ and, for a particular initial condition $\delta p(\cdot,0)$, by the coefficients of its eigendecomposition.

For overparametrized models, we expect the subspace $\mathcal N$ to be nontrivial and possibly even highly degenerate. Assuming that for some ground truth $f$ the minimum loss $L_{\mathrm{mf}}=0$ can be achieved, we can ask how to find the actual limiting distribution $p_\infty$ in the space $p_\infty+\mathcal N$ of all minimizers of $L_{\mathrm{mf}}$. In the linearized setting described above, this should be done using the already mentioned condition that expansion of $p(\cdot,0)-p_\infty$ over the eigenvectors of $\mathcal R$ does not contain the $\lambda=0$ component. In Appendix \ref{sec:lin}, we illustrate this observation with a simple example involving a single-element set $X$.  

\subsection{Models with linear output}\label{sec:linout} We describe now another solvable approximation that holds for models with linear output at small values of the linear parameter. Recall that we defined such models as those where $\phi(\mathbf w,\mathbf x)=w_0\widetilde{\phi}(\widetilde{\mathbf w},\mathbf x)$ with some map $\widetilde{\phi}$ and $\mathbf w=(w_0, \widetilde{\mathbf w})$. Observe that for such models, the gradient factor $\nabla\cdot(p\nabla\phi)$ in the transport equation \eqref{eq:transp} can be written as a sum of two terms: 
\[\nabla_{\mathbf w}\cdot(p\nabla_{\mathbf w}\phi)=\frac{\partial p}{\partial w_0}\widetilde{\phi}+w_0\nabla_{\widetilde{\mathbf w}}\cdot(p\nabla_{\widetilde{\mathbf w}}\widetilde{\phi})\]
For small $w_0$ the second term is small and can be dropped. Then, the transport equation simplifies to 
\begin{equation*}\frac{\partial}{\partial t} p(w_0, \widetilde{\mathbf w}, t)= \widetilde{u}(\widetilde{\mathbf w}, t)\frac{\partial}{\partial {w_0}}p(w_0, \widetilde{\mathbf w},t),\end{equation*}
where 
\begin{align}\widetilde{u}(\widetilde{\mathbf w}, t)={}&\int_X\Big(\int \int p(w'_0, \widetilde{\mathbf w}',t)w'_0\widetilde\phi( \widetilde{\mathbf w}',\mathbf x)dw'_0d\widetilde{\mathbf w}'\nonumber\\ \label{eq:qwt}
&-f(\mathbf x)\Big)\widetilde\phi(\widetilde{\mathbf w}, \mathbf x) d\mu(\mathbf x).
\end{align} 
It follows that the distribution $p$ evolves by ``shifting along $w_0$'', separately at each $\widetilde{\mathbf w}$:
\begin{equation}\label{eq:pwo}p(w_0, \widetilde{\mathbf w}, t)=p(w_0+s(\widetilde{\mathbf w}, t), \widetilde{\mathbf w}, t_0),\end{equation}
where the function $s(\widetilde{\mathbf w}, 0)$ satisfies the equation  
\begin{equation}\label{eq:dts}\frac{\partial}{\partial t} s(\widetilde{\mathbf w}, t)= \widetilde{ u}(\widetilde{\mathbf w}, t)\end{equation}
 with the initial condition $s(\cdot,0)\equiv 0$. 
In particular, for each $\widetilde{\mathbf w}$, the integral $\int p(w_0,\widetilde{\mathbf w},t)dw_0$ -- the marginal of $p$ w.r.t. $\widetilde{\mathbf w}$ -- does not depend on  $t$. We will denote this marginal by $\widetilde p(\widetilde{\mathbf w})$, and by $\widehat p$ we will denote the operator of multiplication by $\widetilde p$, i.e. $\widehat pa(\widetilde{\mathbf w})=\widetilde p(\widetilde{\mathbf w})a(\widetilde{\mathbf w}).$

It is convenient to introduce linear operators $\widetilde \Phi, \widetilde \Phi^*$:
\begin{align*}\widetilde \Phi a({\mathbf x}) ={}&\int \widetilde\phi(\widetilde{\mathbf w},\mathbf x)a(\widetilde{\mathbf w}) d \widetilde{\mathbf w},\\
\widetilde \Phi^* b(\widetilde{\mathbf w}) ={}&\int \widetilde\phi( \widetilde{\mathbf w},\mathbf x)b(\mathbf x) d\mu( {\mathbf x}).
\end{align*}
These operators are mutually adjoint w.r.t. to the scalar products $\langle a_1,a_2\rangle\equiv\int a_1(\widetilde{\mathbf w}) a_2(\widetilde{\mathbf w})d\widetilde{\mathbf w}$ and $\langle b_1,b_2\rangle_\mu\equiv\int_X b_1(\mathbf x) b_2(\mathbf x)d\mu(\mathbf x),$ namely
$\langle \widetilde \Phi a, b\rangle_\mu = \langle  a, \widetilde \Phi^*b\rangle.$

Plugging Eq.\eqref{eq:pwo} into Eq.\eqref{eq:qwt} and performing a change of variables, we obtain
\begin{equation}\label{eq:qwmw}\widetilde{ u}(\widetilde{\mathbf w}, t)=\widetilde{ u}(\widetilde{\mathbf w}, 0)-(\widetilde \Phi^*\widetilde \Phi\widehat p s(\cdot, t))(\widetilde{\mathbf w}).\end{equation}
Denote, for brevity, $\mathbf s(t)=s(\cdot,t),$ and $\widetilde{\mathbf u}_0=\widetilde{ u}(\cdot, 0)$. Combining Eqs.\eqref{eq:qwmw} and \eqref{eq:dts}, we obtain a linear first order equation for $\mathbf s$: 
\begin{equation}\label{eq:dtsq}\frac{d}{dt}\mathbf s= \widetilde{\mathbf u}_0-\widetilde \Phi^*\widetilde \Phi\widehat p\mathbf s.\end{equation} 
The operator $\widetilde{\mathcal K}=\widetilde \Phi^*\widetilde \Phi\widehat p$ is self-adjoint and positive semidefinite with respect to the scalar product $\langle \widehat p\cdot,\cdot\rangle$:
\begin{align*}\langle \widehat p\mathbf a_1,\widetilde \Phi^*\widetilde \Phi\widehat p\mathbf a_2\rangle = \langle \widetilde \Phi\widehat p\mathbf a_1, \widetilde \Phi\widehat p\mathbf a_2\rangle_\mu
={}&\langle \widehat p\widetilde \Phi^*\widetilde \Phi\widehat p\mathbf a_1,\mathbf a_2\rangle.
\end{align*}
Assuming $\mathbf q_0$ belongs to the Hilbert space associated with this scalar product, we can write the solution to Eq.\eqref{eq:dtsq} as
\begin{equation}\label{eq:stt}
\mathbf s(t)=t\mathcal P_{0}\widetilde{\mathbf u}_0+(1-e^{-t\widetilde \Phi^*\widetilde \Phi\widehat p})(\widetilde \Phi^*\widetilde \Phi\widehat p)^{-1}\mathcal P_{>0}\widetilde{\mathbf u}_0
\end{equation}
where $\mathcal P_{0},\mathcal P_{>0}$ are the complementary orthogonal projectors to the nullspace and to the strictly positive subspace of $\widetilde \Phi^*\widetilde \Phi\widehat p$, respectively.

Let $\mathbf z=(z(\mathbf x)),$ where 
\[z(\mathbf x)=\int\int p(w'_0, \widetilde{\mathbf w}', 0)w'_0\widetilde\phi( \widetilde{\mathbf w}',\mathbf x)dw'_0d\widetilde{\mathbf w}'-f(\mathbf x).\]

Using the identity $\widetilde{\mathbf u}_0=\widetilde \Phi^*\mathbf z$ and Eq.\eqref{eq:stt}, the loss function can then be written as
\begin{align*}
L_{\mathrm {mf}}(\mathbf p(t))={}&\frac{1}{2}\| \mathbf z-\widetilde \Phi\widehat p\mathbf s(t)\|_\mu^2\\
={}&\frac{1}{2}\|( 1-\widetilde\Phi\widehat p(\widetilde \Phi^*\widetilde \Phi\widehat p)^{-1}
\mathcal P_{>0}\widetilde \Phi^*) \mathbf z\|_\mu^2\\
&+\frac{1}{2}\langle\widetilde \Phi^* \mathbf z, \widehat p e^{-2t\widetilde \Phi^*\widetilde \Phi\widehat p}(\widetilde \Phi^*\widetilde \Phi\widehat p)^{-1}\mathcal P_{>0}\widetilde \Phi^* \mathbf z\rangle.
\end{align*}
The first, constant term on the r.h.s. is half the squared norm of the component of $\mathbf z$ orthogonal to the range of the operator $\widetilde\Phi\widehat p$. The second term converges to 0 as $t\to +\infty$, so the limit of the loss function is given by the first term. Suppose that the first term vanishes and suppose that the positive semidefinite operator $\widetilde \Phi^*\widetilde \Phi\widehat p$ has a pure point spectrum with eigenvalues $\zeta_k\ge 0$ and eigenvectors $\mathbf s_k$. Then Eq.\eqref{eq:stt} can be expanded as
\begin{equation}\label{eq:st}
\mathbf s(t)=t\mathcal P_{0}\widetilde \Phi^* \mathbf z+\sum_{k:\zeta_k>0}\frac{1-e^{-t\zeta_k}}{\zeta_k}\frac{\langle \widehat p\mathbf s_k,\widetilde \Phi^* \mathbf z\rangle}{\langle \widehat p\mathbf s_k,\mathbf s_k\rangle}\mathbf s_k\end{equation}
and the loss function as
\begin{equation}\label{eq:lmfzeta}L_{\mathrm {mf}}(\mathbf p(t))=\frac{1}{2}\sum_{k:\zeta_k>0}\frac{\langle \widehat p\mathbf s_k,\widetilde \Phi^* \mathbf z\rangle^2}{\langle \widehat p\mathbf s_k,\mathbf s_k\rangle}\frac{e^{-2t\zeta_k}}{\zeta_k}.\end{equation}
Thus, the large-$t$ evolution of loss is determined by the eigenvalues of $\widetilde{\mathcal K}$ and, for a particular $\mathbf z$, by the eigendecomposition of $\Phi^* \mathbf z$.

\subsection{Localized Gaussian aproximation}\label{sec:localgauss}
Suppose that, for each $t$, the distribution $p(\cdot,t)$ is approximately Gaussian with a center $\mathbf b=\mathbf b(t)$ and a small covariance matrix $A=A(t)$:
\begin{equation}\label{eq:gauss}p(\mathbf w,t)=(2\pi)^{-d/2}(\det A)^{-1/2}e^{-\frac{1}{2}(\mathbf w-\mathbf b)\cdot A^{-1}(\mathbf w-\mathbf b)}.\end{equation}
Plugging this ansatz into the transport equation and keeping only leading terms in $A^{-1},$ we obtain (see Appendix \ref{sec:appgauss}): 
\begin{equation}\label{eq:dbt}\frac{d\mathbf b}{dt}\approx-\int_X(\phi(\mathbf b,\mathbf x)-f(\mathbf x))\nabla_{\mathbf w}\phi({\mathbf b}, \mathbf x)d\mu(\mathbf x),\end{equation}
and
\begin{equation}\label{eq:dat}\frac{dA}{dt}\approx AH+HA,\end{equation}
where
\[H=-\int_X(\phi(\mathbf b,\mathbf x)-f(\mathbf x))D_{\mathbf w}\phi({\mathbf b}, \mathbf x)d\mu(\mathbf x)\]
and $D_{\mathbf w}$ denotes the Hessian w.r.t. $\mathbf w.$ Not surprisingly, Eq.\eqref{eq:dbt} coincides with the gradient descent equation for the original model \eqref{eq:approx} with $N=1$. Now consider Eq.\eqref{eq:dat}. If the matrix $H$ is diagonalized,  $H=\operatorname{diag}(\lambda_1,\ldots,\lambda_d),$ then Eq.\eqref{eq:dat} is also diagonalized in the basis of matrix elements:
\[\frac{dA_{km}}{dt}= (\lambda_k+\lambda_m)A_{km}.\]
In particular, our ``pointlike'' Gaussian solution $p$ is unstable (expansive) iff $H$ has positive eigenvalues.

Consider now the special case of a model with linear output, $\phi(\mathbf w,\mathbf x)=w_0\widetilde\phi(\widetilde{\mathbf w},\mathbf x),$ and a distribution $p$ close to a ``pointlike stationary distribution'', so that $\frac{d\mathbf b}{dt}=0$. We write the Hessian and accordingly $H$ in the block form:
\[D_{\mathbf w}=
\begin{pmatrix}
D_{w_0w_0} &
D_{w_0\widetilde{\mathbf w}}\\
D_{w_0\widetilde{\mathbf w}}^t &
D_{\widetilde{\mathbf w}\widetilde{\mathbf w}}\end{pmatrix},\quad 
H=\begin{pmatrix}
H_{w_0w_0} &
H_{w_0\widetilde{\mathbf w}}\\
H_{w_0\widetilde{\mathbf w}}^t &
H_{\widetilde{\mathbf w}\widetilde{\mathbf w}}\end{pmatrix}.\]
Observe that, by linearity of $\phi$ in $w_0$, $D_{w_0w_0}\phi=0$ and hence $H_{w_0w_0}=0$. Also, observe that $D_{w_0\widetilde{\mathbf w}}\phi=\nabla_{\widetilde{\mathbf w}}\widetilde\phi=\frac{1}{w_0}\nabla_{\widetilde{\mathbf w}}\phi.$ It follows that $H_{w_0\widetilde{\mathbf w}}=\frac{1}{b_0}\frac{d \widetilde {\mathbf b}}{dt}$, where we have denoted $\mathbf b=(b_0,\widetilde {\mathbf b}).$ By assumption, $\frac{d \widetilde {\mathbf b}}{dt}=0.$ If  $w_0\ne0$, this implies that $H_{w_0\widetilde{\mathbf w}}=0$. We conclude that the matrix $H$ has the block form   
\begin{equation}\label{eq:h}H=\begin{pmatrix}
0 &
0\\
0 &
H_{\widetilde{\mathbf w}\widetilde{\mathbf w}}\end{pmatrix}.\end{equation}
This shows in particular that a distribution $p$ close to a pointlike stationary distribution evolves only in the $\widetilde{\mathbf w}$-component.      

\section{Application to free-knot linear splines}\label{sec:splines}
In this section we apply the developed general theory to the model of piecewise linear free-knot splines, which can be viewed as a simplified neural network acting on the one-dimensional input space. Specifically, let $X=[0,1],$ $\mu$ be the Lebesgue measure on $[0,1],$ and \begin{equation}\label{eq:spline}\phi(\mathbf w,x)\equiv\phi(c,h,x)=c(x-h)_+,\quad \mathbf w=(c,h)\in\mathbb R^2,\end{equation}
where $a_+\equiv \max(a,0)$ is the ReLU activation function. We will perform numerical simulations for this model with $N=100$ in Eq.\eqref{eq:approx}.

Note that еру model \eqref{eq:spline} is highly degenerate in the sense that the same prediction $\widehat f_{\mathrm {mf}}$  can be obtained from multiple distributions $p$ on the parameter space $\mathbb R^2=\{(c,h)\}$: 
\begin{equation}\label{eq:ffm}\int_{\mathbb R}\int_{\mathbb R}p(c,h)c(x-h)_+ dcdh=\widehat f_{\mathrm {mf}}(x), \quad \forall x\in[0,1].\end{equation}
Using the identity $\frac{d^2}{dh^2}(x-h)_+=\delta(x-h)$ with Dirac delta, we get:
\begin{equation}\label{eq:pcx}\int_{\mathbb R}p(c,x)cdc=\frac{d^2\widehat f_{\mathrm {mf}}}{dx^2}(x), \quad \forall x\in[0,1].\end{equation}
The derivative $\frac{d^2\widehat f_{\mathrm {mf}}}{dx^2}(\cdot)$ determines the prediction $\widehat f_{\mathrm {mf}}$ up to two constants, for example, $\widehat f_{\mathrm {mf}}(0)$ and $\frac{d\widehat f_{\mathrm {mf}}}{dx}(0)$, which can also  be expressed in terms of the distribution $p$:
\begin{align}\label{eq:f0}
-\int_{-\infty}^0\Big(\int_{\mathbb R} p(c,h)cdc\Big)hdh =& \widehat f_{\mathrm {mf}}(0)\\
\label{eq:f'0}
\int_{-\infty}^0\Big(\int_{\mathbb R} p(c,h)cdc\Big)dh  =& \frac{d\widehat f_{\mathrm {mf}}}{dx}(0)\end{align}

Conditions \eqref{eq:pcx} on distributions $p(\cdot,x)$ are independent at different $x$ and leave infinitely many degrees of freedom for these distributions. This property is of course shared by all models with linear output.  

\subsection{Global optima and stationary distributions}\label{sec:stationary}

For a given ground truth $f$, setting $\widehat f_{\mathrm {mf}}=f$ in Eqs.\eqref{eq:pcx}-\eqref{eq:f'0} gives us a criterion for the distribution $p$ to be a global minimizer of the loss $L_{\mathrm{mf}}$. 

More generally, we can ask what are stationary distributions of the dynamics (in the sense of section \ref{sec:mfloss} and proposition \ref{prop:equiv}). Their complete general characterization  is rather cumbersome, so we just consider the particular example of the ground truth function $f(x)=x^2$ (this case is easier thanks to the constant convexity of $f$). One can then show the following properties (see Appendix \ref{sec:astationary}). First, any distribution supported on the subset $\{(c,h):h\ge 1\}$ is stationary. Second, consider the marginal distribution $\widetilde p(h)=\int p(c,h)dc$ and its restriction to the segment $[0,1]$. Then, for a stationary $p$, either  $\operatorname{supp}\widetilde p(h)\cap [0,1]=[0,1]$ and $p$ is a global minimizer, or $\operatorname{supp}\widetilde p(h)\cap [0,1]$ consists of a finite number of equidistantly spaced atoms in $[0,1]$ (and the approximation $\widehat f_{\mathrm {mf}}$ is a piecewise linear spline). In particular, if $\operatorname{supp}\widetilde p(h)$ consists of a single point $h_*$, then $h_*\ge 1$ or $h_*=\frac{\sqrt{6}-1}{5}.$   

We examine in more detail the special case when a stationary measure is a single atom, i.e. $p(c,h)=\delta(c-c_*)\delta(h-h_*)$. There are two possibilities: either $h_*\ge 1$ and then $c_*$ is arbitrary, or $h_*=\frac{\sqrt{6}-1}{5}$ and then $c_*=\frac{4+\sqrt{6}}{5}.$ In Fig.\ref{fig:convpointlike} we show a series of simulations where the initial distribution is an atom, $p(c,h,0)=\delta(c-c_0)\delta(h-h_0)$ with some $c_0,h_0$. We observe such distributions to converge to stationary atomic distributions of one of the above two kinds. 

We consider now ``pointlike'' Gaussian initial distributions $p(c,h,0)=\frac{1}{2\pi\sigma^2}e^{-((c-c_0)^2+(h-h_0)^2)/(2\sigma^2)}$, with a small standard deviation $\sigma.$ Depending on $c_0, h_0,$ we observe two different patterns of evolution of $p$. If the limit of the perfectly atomic distribution ($\sigma=0$) is a stationary point with $h_*=1,$ then the evolution is stable. In contrast, if the limit of the atomic distribution is $(c_*, h_*)=(\frac{4+\sqrt{6}}{5},\frac{\sqrt{6}-1}{5})$ then the evolution is unstable: near this point the weights diverge rapidly along $h$, and then form a curve approaching a globally minimizing distribution spread relatively uniformly over $h\in[0,1],$ see Fig.\ref{fig:stability}. The existence of two patterns is explained by the linear stability study of section \ref{sec:localgauss}: the pattern depends on the sign of $H_{\widetilde{\mathbf w} \widetilde{\mathbf w}}(\equiv H_{hh})$ in Eq.\eqref{eq:h}. By Eq.\eqref{eq:spline} we have $D_{h}\mathbf \phi(c,h,x)=c\delta(h-x)$ and hence, at a stationary point $(c_*, h_*)$,  $H_{hh}=c_*f(h_*)=c_*h_*^2$. Thus, stationary points with $c_*<0$ are stable and contractive in the $h$ direction, while the stationary point  
$(c_*, h_*)=(\frac{4+\sqrt{6}}{5},\frac{\sqrt{6}-1}{5})$ is unstable and expansive in the $h$ direction.

\begin{figure}[h]
\centering
\includegraphics[scale=0.7,clip,trim=3mm 4mm 3mm 9mm]{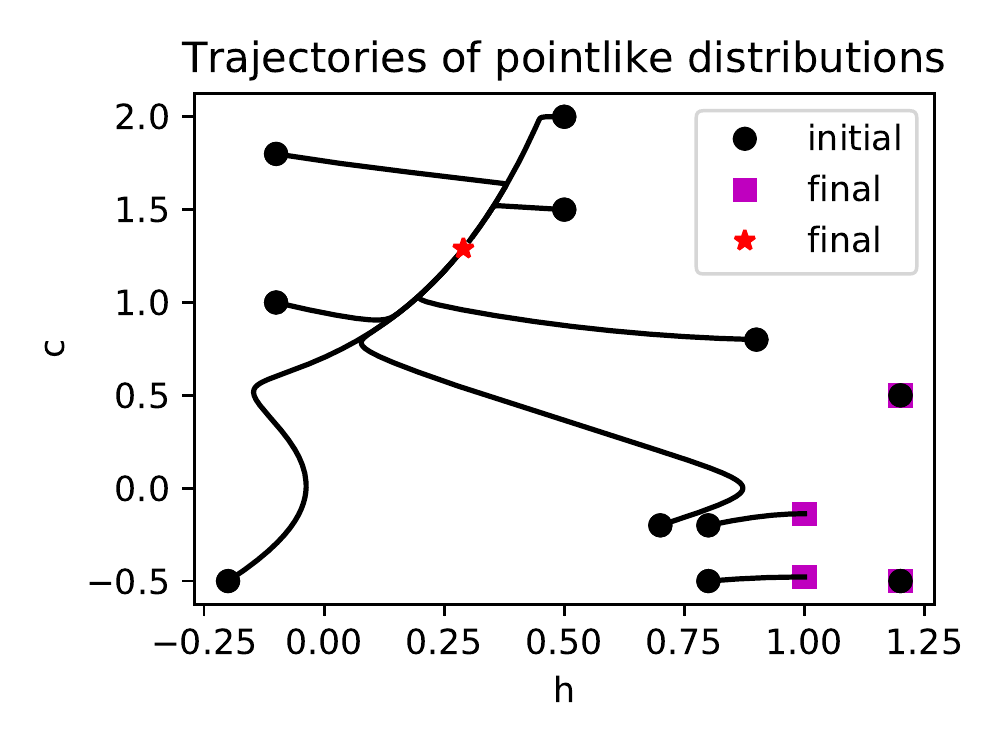}
\caption{Trajectories of various atomic distributions $p(c,h,t)=\delta(c-c(t))\delta(h-h(t))$ converging to stationary atomic distributions.}\label{fig:convpointlike}
\end{figure}

\begin{figure*}[h]
\centering
%
\includegraphics[scale=0.53,clip,trim=10 0 0mm 9mmm]{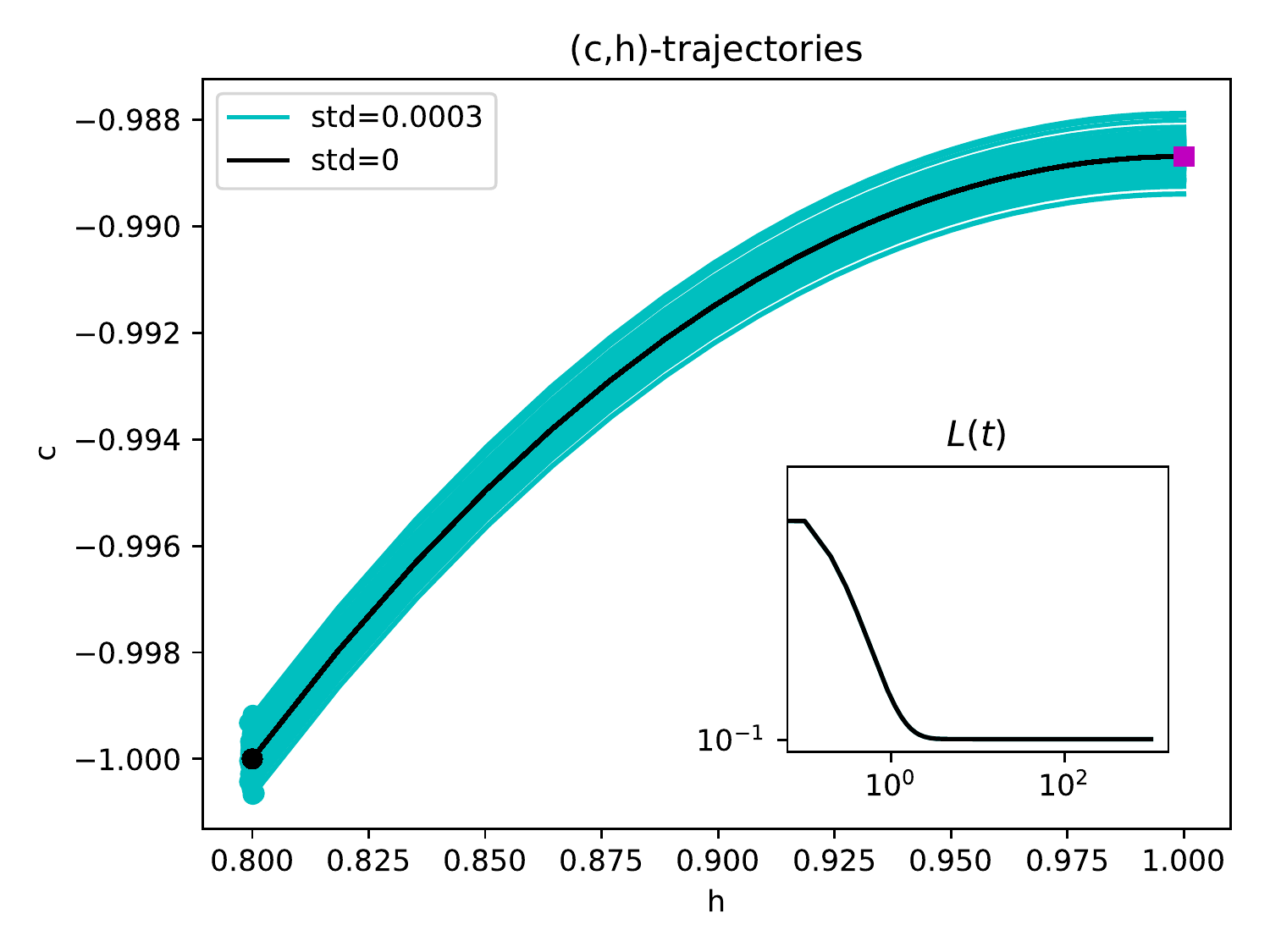}
%
\includegraphics[scale=0.53,clip,trim=0 0 0 9mmm]{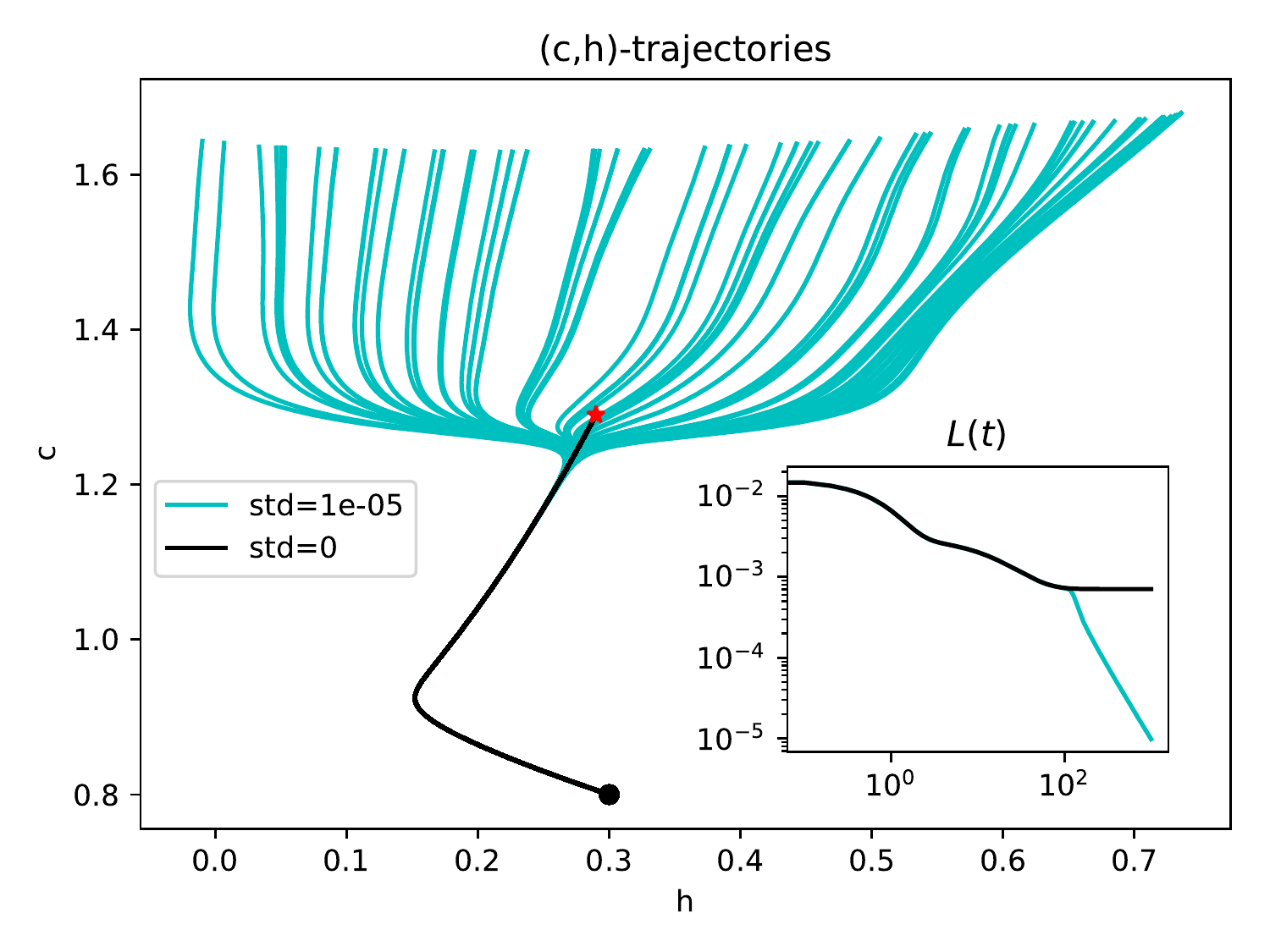}
%
\caption{Trajectories of stable (left) or unstable (right) stationary pointlike distributions. If the initial weights $(c_n,h_n)$ are sampled from a distribution with a small positive variance, in the unstable case the trajectories of the weights spread along the $h$ axis when approaching the stationary point.}\label{fig:stability}
\end{figure*}

\begin{figure*}[h]
\begin{subfigure}[b]{0.32\textwidth}
\includegraphics[scale=0.57,clip,trim=3mm 4mm 0 3mm]{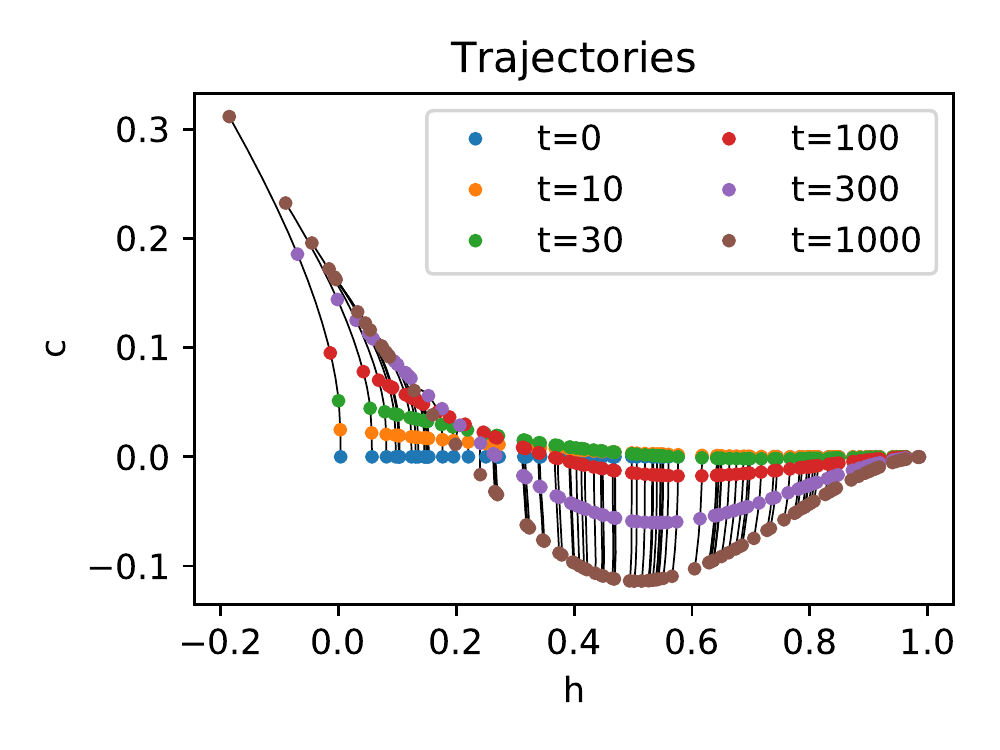}
\caption{}
\end{subfigure}
\begin{subfigure}[b]{0.32\textwidth}
\includegraphics[scale=0.57,clip,trim=0mm 4mm 0 3mm]{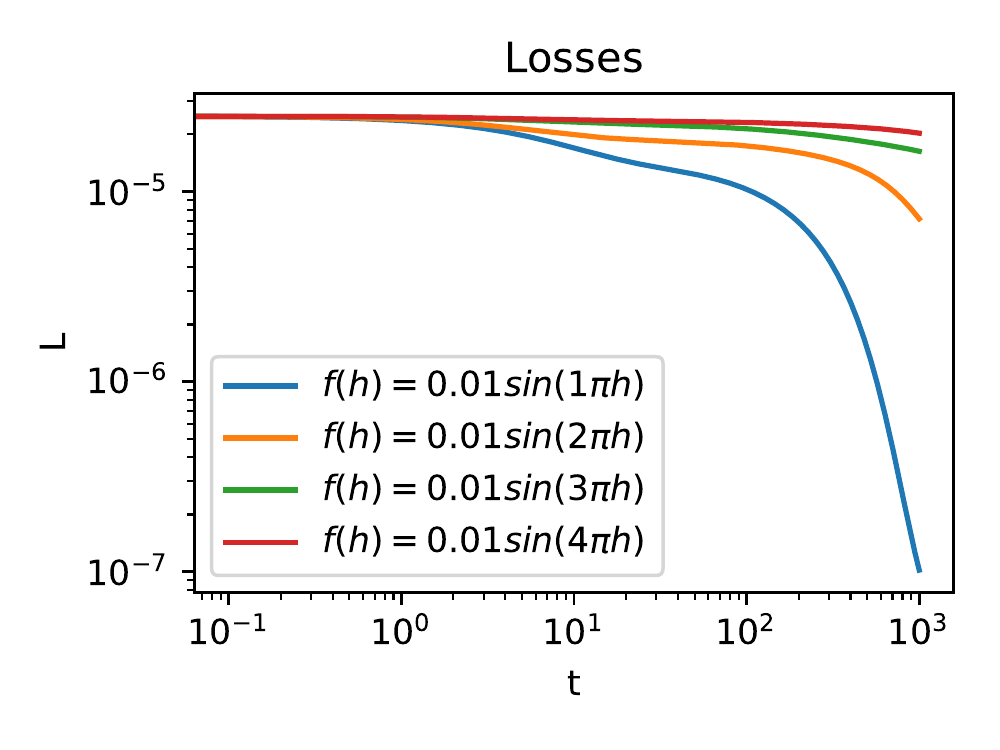}
\caption{}
\end{subfigure}
\begin{subfigure}[b]{0.32\textwidth}
\includegraphics[scale=0.57,clip,trim=-5mm 4mm 6mmm 3mm]{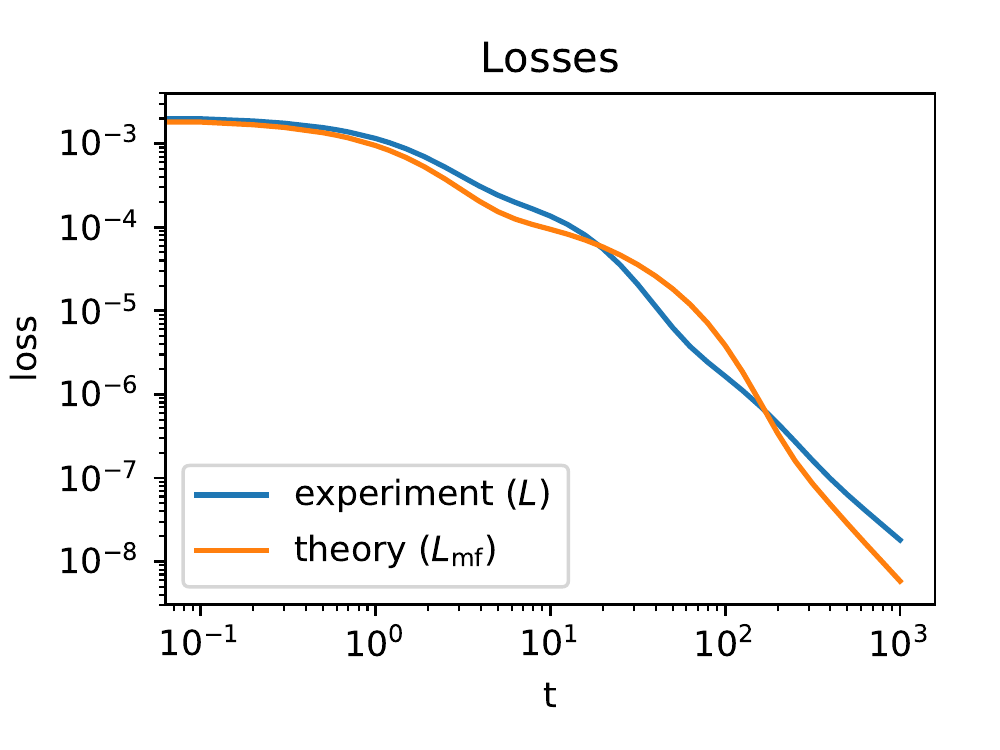}
\caption{}
\end{subfigure}
\caption{(a) The small-$c$ scenario: trajectories of the weights $(c_n,h_n)$. (b) The small-$c$ scenario: experimental losses for several ground truth functions $f$. (c) Linearization near a global minimizer: comparison of experimental loss ($L$) with the theoretical expansion \eqref{eq:lmfptsum} ($L_{\mathrm{mf}}$).}\label{fig:exptheorycompare}
\end{figure*}

\subsection{The small-$c$ approximation}\label{sec:smallc}
We can solve our model approximately in the region of small $c$ as described in section \ref{sec:linout}. Suppose that the initial distribution is $p(c,h,0)=\delta(c)\mathbf 1_{[0,1]}(h).$ As discussed in section \ref{sec:linout}, at small $c$ we expect   $p$ to evolve by approximately shifting along $c$, with some shift function $s(h)$. We observe indeed in an experiment that $p(c,h,t)\approx\delta(c+s(h,t))\mathbf 1_{[0,1]}(h)$, see Fig.\ref{fig:exptheorycompare}(a).  To further understand this dynamics, recall the approximate solutions \eqref{eq:st},\eqref{eq:lmfzeta} (where $\mathbf z=-f$, by our choice of the initial distribution). In our case, $\widehat p$ is the operator of multiplication by the indicator function $1_{[0,1]}(h)$, and the operator $\widetilde {\mathcal K}=\widetilde \Phi^*\widetilde \Phi\widehat p$ can be viewed as a self-adjoint positive definite operator in $L^2([0,1])$ with the kernel
\begin{equation}\label{eq:wtkh}\widetilde K(h,h')=\int_{0}^1(x-h)_+(x-h')_+dx.\end{equation}
This integral operator has a discrete spectrum with eigenvalues $\zeta_k =(1+o(1)) ((\frac{1}{2}+k)\pi)^{-4}, k=0,1,\ldots$, and for large $k$ the eigenfunctions have the form $s_k(h)\approx \cos(\zeta^{-1/4}h+\frac{\pi}{4})$ 
(see Appendix \ref{sec:adiag}). The eigenvalues decrease quite rapidly, e.g. $\zeta_0\approx 0.08, \zeta_1\approx 2\cdot 10^{-3}, \zeta_2\approx 2.6\cdot 10^{-4}$. For any ground truth map $f\in L^2[0,1]$, the loss converges to 0, but convergence is quite slow if the eigendecomposition of $\widetilde \Phi^*f$ contains a significant large-$k$ component. In particular, one can roughly estimate from Eq.\eqref{eq:lmfzeta} that if $\widetilde \Phi^*f=\mathbf s_k,$ then one needs $t\asymp \zeta^{-1}_k\approx (\pi k)^4$ for a noticeable decrease of the loss value. In Fig.\ref{fig:exptheorycompare}(b) we run numerically the gradient descent for several functions $f(h)=0.01\sin(k\pi h)$ and observe indeed a substantial convergence slowdown with growing $k$, in agreement with the  theoretical prediction.

\subsection{Linearization near a global minimizer}\label{sec:splineslinglob}
Now we consider the approximate solution of the loss dynamics near a global optimum (Sec.\ref{sec:linglobmin}) and also make a more careful comparison of the theoretical asymptotic with the numerical solution. To make analysis easier, we consider the simplified network model without the linear weight $c$, i.e. we let $\phi(h,x)=(x-h)_+$. Obviously, this model is less expressive than the full model \eqref{eq:spline}, but it is sufficient for convex ground truths $f$ with $f(0)=\frac{df}{dh}(0)=0$: in this case there is a unique global minimizer $p_\infty(h)=\frac{d^2 f}{dh^2}(h).$ 

We consider a particular example of ground truth, $f(x)=\frac{1}{2}x^2$. Then, the mean field loss has the global minimizer, $p_\infty=\mathbf 1_{[0,1]}$. The operator $\mathcal R$ defined in Eq.\eqref{eq:r} has in this case a discrete spectrum with eigenvalues $\lambda_k=-\mu_k^{-2},k=0,1,\ldots,$ where $\mu_k=\frac{\pi}{2}+k\pi$ (see Appendix \ref{sec:ar}). The corresponding eigenfunctions are $p_k(h)=\sin(\mu_kh)-\frac{1}{\mu_k}\delta(h).$ The mean-field loss expansion \eqref{eq:lmflin} for a solution $p=p_\infty+\delta p$ can be written as
\begin{equation}\label{eq:lmfptsum}L_{\mathrm{mf}}(p(\cdot,t))=\sum_{k=0}^\infty\Big(\int_0^1\sin(\mu_kh)\delta p(h,0)dh\Big)^2\frac{e^{-2t/\mu_k^{2}}}{\mu_k^4}.\end{equation}
We perform a numerical gradient descent with the initial distribution $p(\cdot,0)=2\mathbf 1_{[0.3,0.8]}$ (i.e., the initial weights $h_n$ are randomly chosen in the interval $[0.3,0.8]$). In Fig.\ref{fig:exptheorycompare}(c) we compare the respective experimental loss with the above theoretical prediction (where the series is truncated at $k=10^4$) and observe a reasonable agreement between the two curves. 

\section*{Acknowledgment}
The author thanks Maxim Panov, Anton Zhevnerchuk and Ivan Anokhin for useful discussions. 

\bibliography{../main.bib}

\onecolumn

\appendix

\section{Supplementary material}
\subsection{The continuity equation (Eq.\eqref{eq:conteq})}\label{sec:conteq}
Suppose that the evolution of $\mathbf W$ is governed by the vector field $\mathbf G(\mathbf W)$, i.e. $\frac{d}{dt} \mathbf W = \mathbf G$ as in Eq.\eqref{eq:dtw}. Let $\mathbf W = Q(\mathbf W_0,t)$, i.e. $Q$ gives the solution of this differential equation with the initial condition $\mathbf W(t=0)=\mathbf W_0$ at time $t$. Suppose that we interpret $\mathbf G$ as a probability current, and the probability density is $P(\mathbf W,t)$. Then, the local conservation of probability reads \[\frac{d}{dt}\int_{Q(\Omega,t)}P(\mathbf W,t)d\mathbf W=0\]
for any domain $\Omega.$ Making  a  change of variables,
$$\int_{Q(\Omega,t)}P(\mathbf W,t)d\mathbf W=\int_{\Omega}P(Q(\mathbf W_0,t),t)\det \frac{\partial Q}{\partial \mathbf W_0}d\mathbf W_0,$$
which implies
$$\frac{d}{dt}\Big[P(Q(\mathbf W_0,t),t)\det \frac{\partial Q}{\partial \mathbf W_0}\Big]=0$$
for any $\mathbf W_0$. At $t=0$, using the identity $\frac{d}{dt}\det \frac{\partial Q}{\partial \mathbf W_0}\Big|_{t=0}=\nabla_{\mathbf W}\cdot \mathbf G$, we get
$$\frac{\partial}{\partial t}P+\frac{d}{dt} \mathbf W\cdot\nabla_{\mathbf W}P+P\nabla_{\mathbf W}\cdot \mathbf G=0,$$
which is the desired Eq.\eqref{eq:conteq}.

\subsection{Time derivative of the mean-field loss (Eq.\eqref{eq:dtlossmf})}\label{sec:derivloss}

\begin{align*}
\frac{d}{dt}L_{\mathrm{mf}}(p(\cdot,t)) =& \frac{d}{dt}\frac{1}{2}\int_X \Big(\int p(\mathbf w,t)\phi(\mathbf w,\mathbf x)d\mathbf w-f(\mathbf x)\Big)^2 d\mu(\mathbf x)\\
=&\int_X \Big(\int p(\mathbf w',t)\phi(\mathbf w',\mathbf x)d\mathbf w'-f(\mathbf x)\Big)\Big(\int \frac{\partial}{\partial t}p(\mathbf w,t)\phi(\mathbf w,\mathbf x)d\mathbf w\Big) d\mu(\mathbf x)\\
=&\int_X \Big(\int p(\mathbf w',t)\phi(\mathbf w',\mathbf x)d\mathbf w'-f(\mathbf x)\Big)\Big(\int \big(\nabla_{\mathbf w}\cdot( p(\mathbf w,t)\nabla_{\mathbf w} u(\mathbf w,t))\big)\phi(\mathbf w,\mathbf x)d\mathbf w\Big) d\mu(\mathbf x)
\\
=&-\int_X \Big(\int p(\mathbf w',t)\phi(\mathbf w',\mathbf x)d\mathbf w'-f(\mathbf x)\Big)\Big(\int p(\mathbf w,t)\nabla_{\mathbf w} u(\mathbf w,t))\cdot \nabla_{\mathbf w}\phi(\mathbf w,\mathbf x)d\mathbf w\Big) d\mu(\mathbf x)
\\
=&-\int p(\mathbf w,t)\nabla_{\mathbf w} u(\mathbf w,t))\cdot \Big(\int_X \Big(\int p(\mathbf w',t)\phi(\mathbf w',\mathbf x)d\mathbf w'-f(\mathbf x)\Big)\nabla_{\mathbf w}\phi(\mathbf w,\mathbf x) d\mu(\mathbf x)\Big)d\mathbf w
\\
=&-\int p(\mathbf w,t)\nabla_{\mathbf w} u(\mathbf w,t))\cdot \nabla_{\mathbf w} u(\mathbf w,t))d\mathbf w
\\
=& - \int p(\mathbf w, t) |\nabla_{\mathbf w} u(\mathbf w,t)|^2 d\mathbf w,
\end{align*}
where we used definition \eqref{eq:lmfdef} (first line), Eq.\eqref{eq:transp22} (third line), integrated by parts (fourth line), changed the integration order (fifth line) and used Eq.\eqref{eq:transp21} (sixth line). 

\subsection{Finding the limit distribution in the linearized setting: a basic example (see the end of section \ref{sec:linglobmin})}\label{sec:lin}
Suppose that $X=\{0\}$ is a single-element set with $\mu(0)=1$, $\mathbf w\equiv w\in \mathbb R$, and $\phi(w,0)=w.$ Denote $y=f(0)$. The transport equation \eqref{eq:transp} then simplifies to
\[\frac{\partial}{\partial t} p(w, t)
= \Big(\int  p(w',t)w'd w'-y\Big) 
\frac{\partial}{\partial w}p( w,t).\]
The exact solution with the initial condition $p(\cdot, 0)$ is 
\[p(w, t) = p(w+(y-y_0)(e^{-t}-1),0),\]   
where we have denoted 
\begin{equation}\label{eq:ay0}y_0=\int  p(w',0)w'd w'.\end{equation}
In particular, the limiting distribution is
\begin{equation}\label{eq:apinfw}p_\infty(w)=\lim_{t\to+\infty}p(w,t)=p(w-y+y_0,0).\end{equation}
The loss function evolves as
\begin{equation}\label{eq:almfp}L_{\mathrm{mf}}(p(\cdot,t))=\frac{(y-y_0)^2}{2}e^{-2t}.\end{equation}
Now we check if these exact results can be recovered from the linear approximation about a global minimizer of $L_{\mathrm{mf}}.$ The kernels $K, R$ have the form
\[K(w,w')=ww',\quad R(w,w')=\frac{dp_\infty}{dw}(w)w'.\]  
The subspace $\mathcal N=\{q:\langle Kq,q\rangle=0\}$ has the form \[\mathcal N=\Big\{q:\int q(w)wdw=0\Big\}.\]
The quotient space $\mathcal H'_K=\mathcal H_K/\mathcal N$ is one-dimensional and contains a unique eigenvector of the operator $\mathcal R$. The corresponding eigenvector in $\mathcal H_K$ is $\frac{dp_\infty}{dw}$:
\[\mathcal R \frac{dp_\infty}{dw}(w)=\frac{dp_\infty}{dw}(w)\int \frac{dp_\infty}{dw}(w')w'dw'=\frac{dp_\infty}{dw}(w)\int w'dp_\infty(w')=-\frac{dp_\infty}{dw}(w)\int p_\infty(w')dw'=-\frac{dp_\infty}{dw}(w),\]
i.e., the corresponding eigenvalue is $\lambda_1=-1$. Now we find the limiting distribution $p_\infty$. We know that the difference $p(\cdot,0)-p_\infty$ must be an eigenvector of $\mathcal R$ with a strictly negative eigenvalue, i.e.
\[p(\cdot,0)=p_\infty+c\frac{dp_\infty}{dw}\]
with some constant $c$. We can find this constant using the condition $\int p_\infty(w')w'dw'=y$ and Eq.\eqref{eq:ay0}:
\[y_0-y=\int  (p(w',0)- p_\infty(w'))w'd w'=\int  c\frac{dp_\infty}{dw}(w')w'd w'=-c.\]
It follows that \begin{equation}\label{eq:apw0}p(w,0)=p_\infty(w)+(y-y_0)\frac{dp_\infty}{dw}(w).\end{equation}
Observe that if the distribution $p_\infty$ is sufficiently regular, then we can view the r.h.s. of this equation as the first two terms of the Taylor expansion of $p_\infty(w+y-y_0)$ in the small parameter $y-y_0.$ Thus, within the linear approximation, $p_\infty(w)\approx p(w+y_0-y,0)$, which matches the exact solution \eqref{eq:apinfw} of the transport equation.\footnote{Alternatively, this approximate identity can be derived from the solution $p_\infty(w)=\int_0^\infty p(w+(y_0-y)s,0)e^{-s}ds$ of Eq.\eqref{eq:apw0} by using the linear approximation $p(w+(y-y_0)s,0)\approx p(w,0)+(y-y_0)s\frac{\partial p}{\partial w}(w,0)$.}

The approximate dynamics  of the loss function is given by expansion \eqref{eq:lmflin} that in the present case consists of a single term:
\[L_{\mathrm{mf}}(p(\cdot,t))\approx\frac{1}{2}\frac{\langle\mathcal K [(y-y_0)\frac{dp_\infty}{dw}],\frac{dp_\infty}{dw}\rangle^2}{\langle\mathcal K \frac{dp_\infty}{dw},\frac{dp_\infty}{dw}\rangle}e^{2\lambda_1t}=\frac{(y-y_0)^2}{2}e^{-2t},\]
which agrees precisely with the exact value  \eqref{eq:almfp}.

\subsection{Localized Gaussian approximation: derivation of Eqs.\eqref{eq:dbt} and \eqref{eq:dat}}\label{sec:appgauss}
If $p$ has the Gaussian form \eqref{eq:gauss}, we have
\begin{equation}\label{eq:adtlnp}\frac{\partial}{\partial t}\ln p(\mathbf w,t)=\frac{d}{dt}\Big(\frac{1}{2}\ln \det A\Big)+\frac{d\mathbf b}{dt}\cdot A^{-1}(\mathbf w-\mathbf b)+\frac{1}{2}(\mathbf w-\mathbf b)\cdot A^{-1}\frac{dA}{dt}A^{-1}(\mathbf w-\mathbf b).\end{equation}
On the other hand, from the transport Eqs.\eqref{eq:transp21},\eqref{eq:transp22},
\begin{align}\nonumber
\frac{\partial}{\partial t}\ln p(\mathbf w,t)={}&\frac{1}{p(\mathbf w,t)}\nabla_{\mathbf w}\cdot(p(\mathbf w,t)\nabla_{\mathbf w}u(\mathbf w, t))\\\nonumber
={}&\nabla_{\mathbf w}\ln p(\mathbf w,t)\cdot\nabla_{\mathbf w}u(\mathbf w, t)+\Delta_{\mathbf w}u(\mathbf w, t)\\\nonumber
={}& \int_X\Big(\int p(\mathbf w',t)\phi(\mathbf w',\mathbf x)d\mathbf w'-f(\mathbf x)\Big)\Big(-A^{-1}(\mathbf w-\mathbf b)\cdot\nabla_{\mathbf w}\phi({\mathbf w}, \mathbf x)+\Delta_{\mathbf w}\phi({\mathbf w}, \mathbf x)\Big)d\mu(\mathbf x)\\\nonumber
\approx{}&\int_X\Big(\phi(\mathbf b,\mathbf x)-f(\mathbf x)\Big)\Big(-A^{-1}(\mathbf w-\mathbf b)\cdot\nabla_{\mathbf w}\phi({\mathbf w}, \mathbf x)+\Delta_{\mathbf w}\phi({\mathbf w}, \mathbf x)\Big)d\mu(\mathbf x)\\\nonumber
\approx{}&\int_X(\phi(\mathbf b,\mathbf x)-f(\mathbf x))\Big(-A^{-1}(\mathbf w-\mathbf b)\cdot\nabla_{\mathbf w}\phi({\mathbf w}, \mathbf x)\Big)d\mu(\mathbf x)\\
\approx{}&\int_X(\phi(\mathbf b,\mathbf x)-f(\mathbf x))\Big(-A^{-1}(\mathbf w-\mathbf b)\cdot\big(\nabla_{\mathbf w}\phi({\mathbf b}, \mathbf x)+D_{\mathbf w}\phi({\mathbf b}, \mathbf x)(\mathbf w-\mathbf b)\big)\Big)d\mu(\mathbf x),\label{eq:adtlnp2}
\end{align}
where $D_{\mathbf w}$ denotes the Hessian w.r.t. $\mathbf w$. Here, in line 4 we used the approximation $p(\mathbf w,t)\approx \delta(\mathbf w-\mathbf b)$, in line 5 we dropped the term $\Delta_{\mathbf w} \phi$ as small w.r.t $A^{-1}$, and in line 6 we linearly expanded $\nabla_{\mathbf w}\phi$ about $\mathbf w=\mathbf b$.

Now, we compare the terms in Eqs.\eqref{eq:adtlnp},\eqref{eq:adtlnp2} that are linear or quadratic in $\mathbf w-\mathbf b$. Comparison of the linear terms gives us
\[\frac{d\mathbf b}{dt}=-\int_X(\phi(\mathbf b,\mathbf x)-f(\mathbf x))\nabla_{\mathbf w}\phi({\mathbf b}, \mathbf x)d\mu(\mathbf x).\]
Comparison of the quadratic terms gives us
\[\frac{dA}{dt}=AH+HA,\]
where
\[H=-\int_X(\phi(\mathbf b,\mathbf x)-f(\mathbf x))D_{\mathbf w}\phi({\mathbf b}, \mathbf x)d\mu(\mathbf x).\]

\subsection{Stationary distributions of the linear spline model (section \ref{sec:stationary})}\label{sec:astationary}
We aim to describe stationary distributions  of the transport equation associated with the spline model \eqref{eq:spline}.  To this end, it is convenient to apply condition 2) of proposition \ref{prop:equiv}, observing that in our spline model $\nabla_{\mathbf w} \phi(\mathbf w, x)=\nabla_{(c,h)} \phi(c,h, x)=((x-h)_+, -c\mathbf 1_{[h,+\infty)})$. Using the definition \eqref{eq:transp21} of $u$, we obtain this criterion: a distribution $p$ is stationary iff for any $(c,h)\in\operatorname{supp} p\cap\{c\ne 0, h<1\}$ the corresponding error function $\widehat f_{\mathrm {mf}}-f$ is orthogonal to linear functions in $L^2([h_+,1])$, and for any $(0,h)\in\operatorname{supp} p\cap\{h<1\}$ the error function $\widehat f_{\mathrm {mf}}-f$ is orthogonal to $g(x)=x-h$ in $L^2([h_+,1])$. Here $\operatorname{supp} p$ denotes the (closed) support of the distribution $p$, and $\widehat f_{\mathrm {mf}}$ is given by \eqref{eq:ffm}.

Using this criterion, we can make several observations:
\begin{enumerate}
\item Any distribition supported on the set $\{(c,h):h\ge 1\}$ is stationary. 
\item Let $p$ be a stationary distribution with the corresponding prediction $\widehat f_{\mathrm {mf}}$. Assume that $f$ is continuous and $p$ is sufficiently regular so that $\widehat f_{\mathrm {mf}}$ is also continuous. Let $\widetilde p_{-0}(h)=\int_{\mathbb R\setminus \{0\}}p(c,h)dc$.  Suppose that $h\in \operatorname{supp}\widetilde p_{-0}\cap(0,1)$, and that $h$ is not an isolated point of $\operatorname{supp}\widetilde p_{-0}$. Then, $\widehat f_{\mathrm {mf}}(h)=f(h).$ Indeed, since $h$ is not isolated, we can choose a sequence $h_k\to h, h_k\in \operatorname{supp}\widetilde p_{-0}.$ Then $\widehat f_{\mathrm {mf}}-f$ is orthogonal to all linear functions, and in particular to the constant function, in $L^2([h,1])$ and in $L^2([h_k,1])$, and hence in $L^2([h,h_k])$. Since $h_k\to h$, we must have $(\widehat f_{\mathrm {mf}}-f)(h)=0$ by the continuity of $\widehat f_{\mathrm {mf}}-f$.
\item Let $p$ be stationary and the marginal $\widetilde p_{-0}(h)$ be defined as above. Suppose that $h_1,h_2\in[0,1]$ are two points of $\operatorname{supp}\widetilde p_{-0}$, and $(h_1,h_2)\cap \operatorname{supp}\widetilde p_{-0}=\emptyset$. Then, on the interval $[h_1,h_2]$, the function $\widehat f_{\mathrm {mf}}$  is a linear regression of the function $f$ in the usual $L^2([h_1,h_2])$ sense.
\end{enumerate}
Let us now consider a particular function, say $f(x)=x^2$, and use these observation to explicitely classify all stationary distributions $p$ in this case. First, note that a stationary distribution may have an arbitrary component supported on the set $\{(c,h):h\ge 1\}$ and making no effect on the prediction; so it suffices to consider only distributions $p$ such that $p(\{(c,h):h\ge 1\})=0.$ 

It is convenient to consider separately the cases when  the distribution $p$ has or does not have a component in the halfplane $\{(c,h):h< 0\}$. 

\emph{Case 1: $p$ does not have a component in the halfplane $\{(c,h):h< 0\}$.}
Consider again the marginal $\widetilde p_{-0}(w)=\int_{\mathbb R\setminus \{0\}}p(c,w)dc$. The complement to the support of $\widetilde p_{-0}$  can be written as a countable union of nonoverlapping open intervals: 
\begin{equation}\label{eq:aak}\mathbb R\setminus \operatorname{supp} \widetilde p_{-0}=\cup_k(a_k,b_k).\end{equation} 
Suppose that for some $k$ we have $a_k,b_k\in[0,1].$ Then, by observation 3 above, $\widehat f_{\mathrm {mf}}$ is a linear regression of $f$ on the interval $(a_k,b_k)$. Since $f(x)=x^2$, this implies, in particular, that
 \begin{equation}\label{eq:aab}f(a_k)-\widehat f_{\mathrm {mf}}(a_k)=f(b_k)-\widehat f_{\mathrm {mf}}(b_k)=\frac{(b_k-a_k)^2}{6}.\end{equation}
   This shows, by observation 2 above, that $a_k,b_k$ are isolated points of $\operatorname{supp} \widetilde p_{-0}.$ Similarly, if $a_k\in(0,1)$ and $b_k=+\infty$, then $a_k$ is an isolated point of $\operatorname{supp} \widetilde p_{-0}$ and $\widehat f_{\mathrm {mf}}$ is a linear regression of $f$ on the interval $(a_k,1)$, and
\begin{equation}\label{eq:aab2}f(a_k)-\widehat f_{\mathrm {mf}}(a_k)=f(1)-\widehat f_{\mathrm {mf}}(1)=\frac{(1-a_k)^2}{6}.\end{equation} 
  If the left end $a_k$ of an interval is an isolated point of $\operatorname{supp} \widetilde p_{-0}$, then it must be equal to the right end $b_m$ of another interval $(a_m,b_m)$ from expansion \eqref{eq:aak}. We conclude that if there is at least one interval $(a_k,b_k)$ with $a_k\in(0,1),$ then all other intervals can be reconstructed, one-by-one, by considering neighboring intervals. Thanks to identities \eqref{eq:aab},\eqref{eq:aab2}, $\operatorname{supp} \widetilde p_{-0}$ must then be a finite set consisting of equally spaced points $h_1=1-\Delta h, h_2=1-2\Delta h,\ldots, h_M=1-M\Delta h$, and $f(h_k)-\widehat f_{\mathrm {mf}}(h_k)=\frac{(\Delta h)^2}{6}.$ Since $h_M$ is the leftmost point of $\operatorname{supp} \widetilde p_{-0}$, we have $\widehat f_{\mathrm {mf}}(h_M)=0$. Then, the value of $\Delta h$ can be computed for a given $M$ from the quadratic equation $\frac{(\Delta h)^2}{6}=f(h_M)=(1-M\Delta h)^2$,  specifically, $\Delta h=\frac{6M-\sqrt{6}}{6M^2-1}.$ In particular, for $M=1$ we obtain $h_1=\frac{\sqrt{6}-1}{5}$.      

The only possibility for $\widetilde p_{-0}$ to have non-isolated points is if $\operatorname{supp} \widetilde p_{-0}=[a,1]$ with some $a\in[0,1)$ (so that there is no interval $(a_k,b_k)$ with $a_k\in(0,1)$).  Then, by observation 2, we have $\widehat f_{\mathrm {mf}}=f$ on $[a,1],$ and this can only be satisfied if $a=0$. So, in this case $p$ must be a globally optimal distribution.

\emph{Case 2: $p$ has a component in the halfplane $\{(c,h):h< 0\}$.} This case can be analized similarly, with the difference that if $a$ is the leftmost point of $\operatorname{supp} \widetilde p_{-0}$ in the interval $(0,1),$ then  $\widehat f_{\mathrm {mf}}$ is a linear regression of $f$ on the interval $[0,a].$ Like before, the are two possibilities: either the distribution $p$ is globally optimal and $\operatorname{supp} \widetilde p_{-0}$ covers the whole segment $[0,1]$, or the marginal distribution $\widetilde p_{-0}$ is discrete in $[0,1]$. In the second case, $\widetilde p_{-0}$ has finitely many atoms equally spaced between 0 and 1. 

We make a couple of further remarks. 
\begin{enumerate}
\item The conclusion that $\operatorname{supp} \widetilde p_{-0}\cap [0,1]$ is either a finite set or a segment $[a,1]$ extends to any uniformly convex function $f$. 
\item For $f(x)=x^2$ and a stationary $p$, the supports of the full marginal $\widetilde p(w)=\int_{\mathbb R}p(c,w)dc$ and of $\widetilde p_{-0}(w)=\int_{\mathbb R\setminus \{0\}}p(c,w)dc$ coincide on the interval $(0,1)$. 
\end{enumerate}

\subsection{Diagonalization of the operator $\widetilde {\mathcal K}$ with kernel \eqref{eq:wtkh} (section \ref{sec:smallc}) }\label{sec:adiag}
Consider the eigenvector equation $\zeta s(h)=\widetilde {\mathcal K}s(h)$ for the operator $\widetilde {\mathcal K}$ with kernel \eqref{eq:wtkh}, 
i.e.
\begin{equation}\label{eq:aazs}\zeta s(h)=\int_{0}^1\Big(\int_0^1 (x-h)_+(x-h')_+dx\Big)s(h')dh'.\end{equation}
Differentiating twice this equation in $h$ and using the identity $\frac{d^2}{dh^2}(x-h)_+=\delta(x-h),$ we obtain
\begin{equation}\label{eq:azfd}\zeta\frac{d^2}{dh^2}s(h)=\int_{0}^1(h-h')_+s(h')dh'.\end{equation}
Differentiating again and using the same identity, we get
\begin{equation}\label{eq:azfd4}\zeta\frac{d^4}{dh^4}s(h)=s(h).\end{equation}
Moreover, Eq.\eqref{eq:aazs} implies the boundary conditions $s(1)=\frac{d}{dh}s(1)=0$ and Eq.\eqref{eq:azfd} implies the boundary conditions $\frac{d^2}{dh^2}s(0)=\frac{d^3}{dh^3}s(0)=0.$ It follows from Eq.\eqref{eq:azfd4} and the boundary conditions at $h=0$ that an eigenvector must have the form $s(h)=ag_1(h)+bg_2(h)$ with some coefficients $a,b$ and
\[g_1(h)=\cosh(\xi h)+\cos(\xi h),\quad g_2(h)=\sinh(\xi h)+\sin(\xi h),\quad 0<\xi=\zeta^{-1/4}.
\]
The boundary conditions at $h=1$ can be satisfied if \[\det\begin{pmatrix}g_1(1)&g_2(1)\\\frac{dg_1}{dh}(1)&\frac{dg_1}{dh}(1)\end{pmatrix}=0,\]
which gives the condition on $\xi:$
\[\cosh(\xi)\cos(\xi)=-1.\]
The solutions to this equations can be written as
\[\xi_k=\frac{\pi}{2}+\pi k+2(-1)^ke^{-\pi/2-\pi k}(1+o(1)),\quad k=0,1,\ldots\]
In particular, the first several values $\xi_k$ and the respective eigenvalues $\zeta_k=\xi_k^{-4}$ are
\begin{align*}
\xi_0&\approx 1.87510, \quad \zeta_0\approx 0.08089\\
\xi_1&\approx 4.69409, \quad \zeta_1\approx 0.002059\\
\xi_2&\approx 7.85475, \quad \zeta_2\approx 0.0002627. 
\end{align*}
An eigenfunction for the eigenvalue $\zeta_k$ can then be written as 
\begin{equation}\label{eq:askh}s_k(h)=\cosh(\xi_k h)+\cos(\xi_k h)-\frac{\cosh(\xi_k)+\cos(\xi_k)}{\sinh(\xi_k)+\sin(\xi_k)}(\sinh(\xi_k h)+\sin(\xi_k h)).\end{equation}
Several first eigenfunctions are shown in Fig.\ref{fig:eigenf}.  
\begin{figure}
\begin{center}
\includegraphics[scale=0.5]{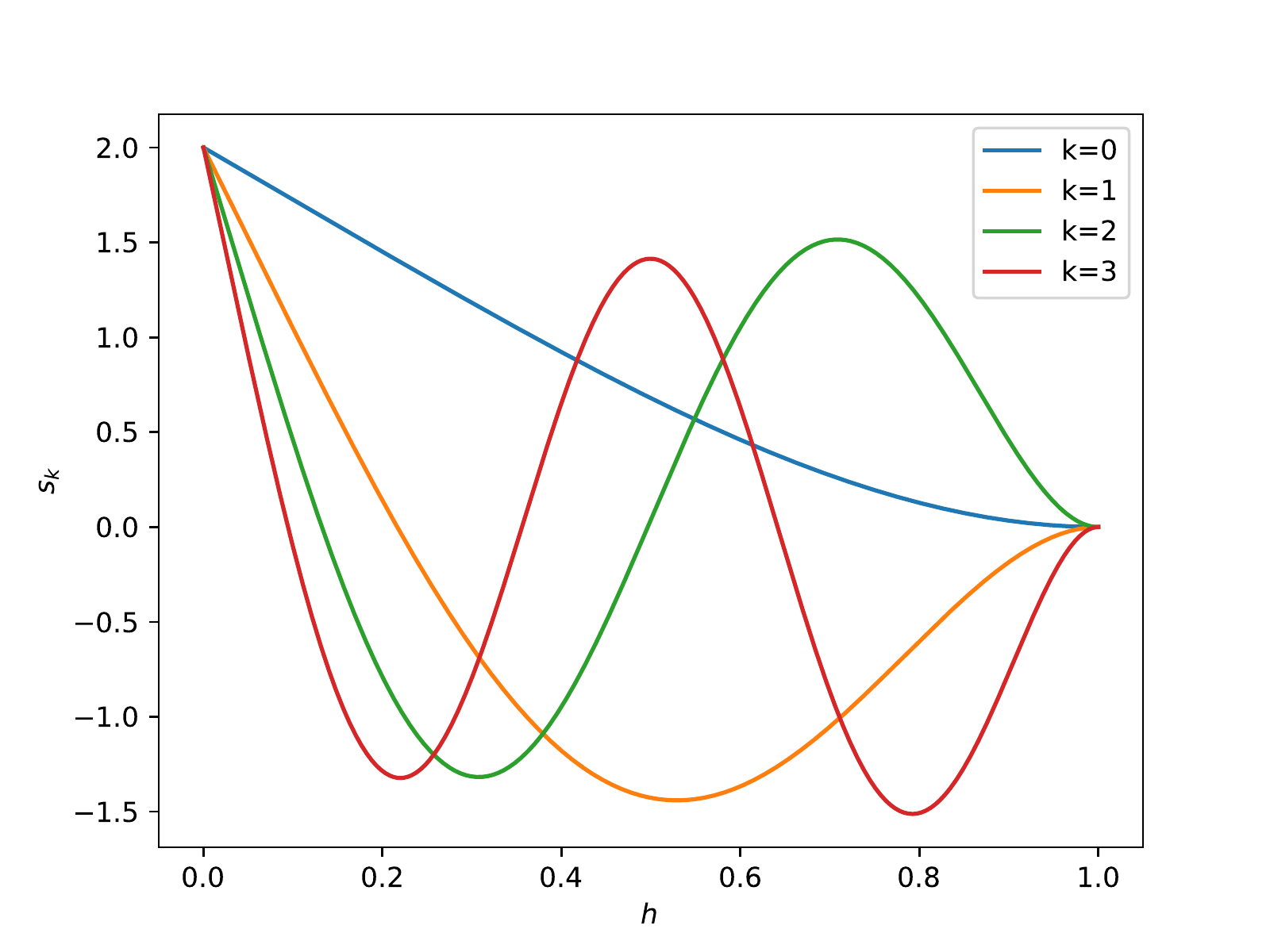}
\caption{Eigenfunctions \eqref{eq:askh} of the operator $\widetilde {\mathcal K}$.}\label{fig:eigenf}
\end{center}
\end{figure}
In the limit $k\to\infty$ we have
\[\frac{\cosh(\xi_k)+\cos(\xi_k)}{\sinh(\xi_k)+\sin(\xi_k)}=\frac{e^{\xi_k}/2}{e^{\xi_k}/2+(-1)^k}+o(e^{-\xi_k})=1-2(-1)^ke^{-\xi_k}+o(e^{-\xi_k}),\]
so
\begin{align*}s_k(h)={}&\cos(\xi_k h)-\sin(\xi_k h)+e^{-\xi_kh}+(-1)^ke^{-\xi_k(1-h)}+o(1)\\
\approx{}&\sqrt{2}\cos\Big(\xi_k h+\frac{\pi}{4}\Big). 
\end{align*}
%
\begin{figure}
\begin{center}
\includegraphics[scale=0.65,trim=3mm 3mm 0mm 0mm,clip]{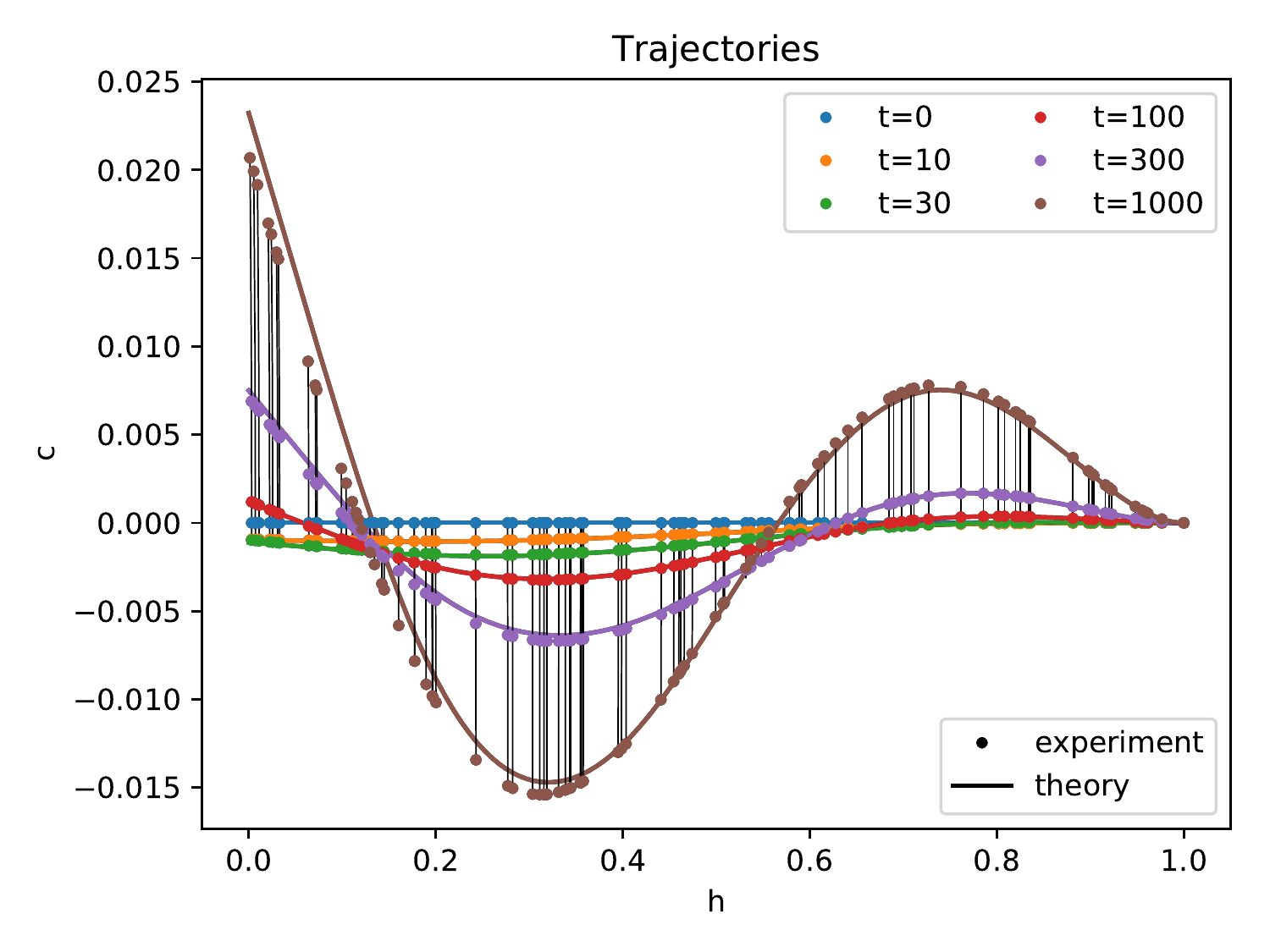}
\includegraphics[scale=0.65,trim=0mm 3mm 3mm 0mm,clip]{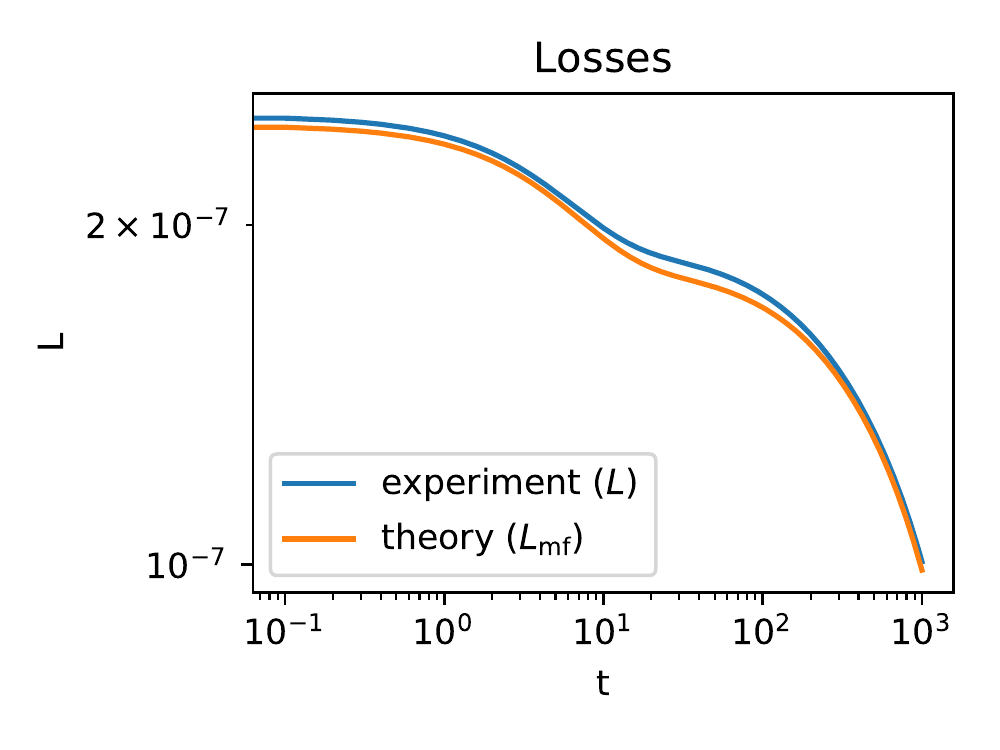}
\caption{Comparison of directly simulated \eqref{eq:dtw} and theoretically described \eqref{eq:st},\eqref{eq:lmfzeta} evolutions of the weights (left) and of the loss (right) in the small-$c$ scenario for the ground truth function $f(x)=10^{-3}\sin 2\pi x$. The scalar products $\langle \widehat p\mathbf s_k,\widetilde \Phi^* \mathbf z\rangle,\langle \widehat p\mathbf s_k,\mathbf s_k\rangle$ are found numerically; the expansions  \eqref{eq:st},\eqref{eq:lmfzeta} are truncated at $k=10$.}\label{fig:smallC_lossCompare}
\end{center}
\end{figure}
In Fig.\ref{fig:smallC_lossCompare} we use the obtained spectral decomposition to compute a few first terms in expansions \eqref{eq:st},\eqref{eq:lmfzeta} for the ground truth function $f(x)=10^{-3}\sin 2\pi x$. We compare the results with direct simulation and observe a good agreement between them.

\subsection{Diagonalization of the operator $\mathcal R$ (section \ref{sec:splineslinglob})}\label{sec:ar}
We diagonalize the operator $\mathcal R$ appearing in Sec.\ref{sec:splineslinglob} and describing dynamics linearized about the global minimizer $p_\infty=\mathbf 1_{[0,1]}$ of the mean-field loss function. 

By Eq.\eqref{eq:rk}, $\mathcal R$ can be viewed as the integral operator with the kernel $R(h,h')=\frac{d}{dh}(p_\infty(h)\frac{d}{dh}) K( h, h'),$ where $K(h,h')$ is the kernel that has already appeared previously, \[K(h,h')=\int_0^1 (x-h)_+(x-h')_+dx.\]
It follows that 
\begin{align*}
R(h,h')
={}&(h-h')_+\mathbf 1_{[0,1]}(h)+\delta(h)\frac{d}{dh} K(0,h')-\delta(h-1)\frac{d}{dh} K(1,h')\\
={}& (h-h')_+\mathbf 1_{[0,1]}(h)-\delta(h)\int_{0}^1 (x-h')_+ dx,
\end{align*}
with Dirac delta. Note that $\int_{\mathbb R}R(h,h')dh = 0$ for all $h'$. The operator $\mathcal R$ creates a $\delta$-measure at $h=0$ that compensates the ``regular'' output component, so that \begin{equation}\label{eq:aintr}\int_{\mathbb R}\mathcal R p(h)dh=0\end{equation}
for any $p$. (Note that the $\delta$-measure has a finite norm $\langle\mathcal K\cdot,\cdot\rangle$.)

Consider now the eigenvalue equation
\[\lambda p = \mathcal Rp \quad (\lambda\le 0).\]
The operator $\mathcal R$ has an infinite-dimensional nullspace (for example, any distribution $p$ supported on $[1,+\infty)$ is in the nullspace). We will be interested in distributions converging to $p_\infty$, so (as discussed in section \ref{sec:linglobmin}) in what follows we will only be interested in strictly negative eigenvalues. Since $\mathcal R$ creates Dirac delta at $h=0$, we look for an eigenfunction in the form \[p=p_{\mathrm{reg}}+c\delta,\] where $p_{\mathrm{reg}}$ is a ``regular'' component of $p$. Then, be Eq.\eqref{eq:aintr}, \[c=-\int_{0}^1p_{\mathrm{reg}}(h)dh.\] When restricted to the regular component, the eigenvalue equation reads:
\begin{align*}
\lambda p_{\mathrm{reg}}(h)=&\int_0^1(h-h')_+p(h')dh' \\
= & \int_0^1(h-h')_+p_{\mathrm{reg}}(h')dh'+ch\\
= & \int_0^h(h-h')p_{\mathrm{reg}}(h')dh'+ch.
\end{align*}
It follows that 
\[\lambda p''_{\mathrm{reg}}=p_{\mathrm{reg}},\quad p_{\mathrm{reg}}(0)=0,\quad \lambda p_{\mathrm{reg}}'(0)=c.\]
Hence $p_{\mathrm{reg}}(h)=\sin(\mu h)$ and $\lambda=-\mu^{-2}$ for some $\mu$. Moreover, $-\mu^{-2}\mu=\lambda p_{\mathrm{reg}}'(0)=c$, i.e. $c=-\frac{1}{\mu}$. Thus $-\frac{1}{\mu}=-\int_0^1\sin (\mu h) dh=\frac{\cos \mu-1}{\mu},$ i.e. $\cos\mu=0$ and hence $\mu=\frac{\pi}{2}+k\pi, k=0,1,\ldots$ 

Summarizing, $\mathcal R$ has eigenvectors $p_k$ and eigenvalues $\lambda_k,$ where
\[ p_k(h) = \sin (\mu_kh)-\frac{1}{\mu_k}\delta(h),\quad \mu_k=\frac{\pi}{2}+k\pi, \quad \lambda_k=-\mu_k^{-2}.\]
Let us compute $\mathcal Kp_k:$
\begin{align*}
\mathcal Kp_k(h) ={}& \int_{0}^1(x-h)_+\Big(\int_0^1(x-h')_+\big(\sin (\mu_kh')-\frac{1}{\mu_k}\delta(h')\big)dh'\Big)dx \\
={}&\frac{1}{\mu_k}\int_{0}^1(x-h)_+\Big(x-(x-h')\cos(\mu_kh')|_{h'=0}^x-\int_0^x\cos(\mu_kh')dh'\Big)dx\\
={}&-\frac{1}{\mu_k^2}\int_{0}^1(x-h)_+\sin(\mu_kx)dx\\
={}&-\frac{1}{\mu_k^3}\Big(-(x-h)\cos(\mu_kx)|_{x=h}^1+\int_{h}^1\cos(\mu_kx)dx\Big)\\
={}&\frac{1}{\mu_k^4}(\sin(\mu_kh)-(-1)^k).
\end{align*}
In particular,
\begin{equation}\label{eq:akpkpn}\langle \mathcal K p_k, p_n\rangle=\frac{1}{\mu_k^4}\int_0^1(\sin(\mu_kh)-(-1)^k)(\sin (\mu_nh)-\frac{1}{\mu_k}\delta(h))dh=\frac{1}{2\mu_k^4}\delta_{kn}.\end{equation}
Also, for any $p$ such that $\int_0^1p(h)dh=0,$ 
\begin{equation}\label{eq:akpkp}\langle \mathcal K p_k, p\rangle=\frac{1}{\mu_k^4}\int_0^1\sin(\mu_kh)p(h)dh.\end{equation}
Any $p$ supported on $[0,1]$ and such that $\int_0^1p(h)dh=0$ can be expanded over the eigenvectors $p_k$. Let $p(h,t)=p_\infty(h)+\delta p(h,t)$ be a solution of the transport equation such that $p(\cdot,0)$ is supported on $[0,1]$ and $\int_0^1p(h,0)dh=\int_0^1p_\infty(h)dh=1.$ Then $\delta p(\cdot,0)$ can be expanded over the eigenvectors $p_k$ and hence, by Eqs.\eqref{eq:akpkpn},\eqref{eq:akpkp}  the loss expansion \eqref{eq:lmflin} can be written as
\begin{align*}L_{\mathrm{mf}}(p(\cdot,t))={}& \sum_{k=0}^\infty\Big(\int_0^1\sin(\mu_kh)\delta p(h,0)dh\Big)^2\frac{e^{-2t/\mu_k^{2}}}{\mu_k^4}.\end{align*}  

\end{document}